\documentclass[11pt]{article}

\usepackage[utf8]{inputenc} 

\pdfoutput=1

\usepackage{algorithm}
\usepackage{algorithmic}
\usepackage{amsmath}
\usepackage{amssymb}
\usepackage{amsthm,amsfonts}
\usepackage{array}
\usepackage{authblk}
\usepackage{babel}
\usepackage{booktabs}
\usepackage{caption}
\usepackage{colortbl}
\usepackage{enumitem}
\usepackage[T1]{fontenc}    
\usepackage{framed}
\usepackage{fullpage}
\usepackage{graphicx}
\usepackage[colorlinks, linkcolor=red, anchorcolor=blue, citecolor=blue]{hyperref}
\usepackage{indentfirst}
\usepackage[utf8]{inputenc}
\usepackage{lipsum}        
\usepackage{makecell}
\usepackage{mathrsfs}
\usepackage{mathtools}

\usepackage[framemethod=default]{mdframed}
\usepackage{microtype}     
\usepackage{multicol}
\usepackage{multirow}
\usepackage{natbib}
\usepackage{nicefrac}       
\usepackage{tablefootnote}
\usepackage{url}            
\usepackage{wrapfig,lipsum}
\usepackage{xcolor}         
\usepackage{xspace}

\usepackage{doi}
\usepackage{tikz}
\usetikzlibrary{arrows.meta}
\usetikzlibrary{decorations.pathreplacing}
\usepackage{amsmath,amsfonts,bm}

\def\1{\bm{1}}

\def\rvx{{\mathbf{x}}}
\def\rvy{{\mathbf{y}}}

\def\vf{{\bm{f}}}

\def\vs{{\bm{s}}}

\def\vx{{\bm{x}}}
\def\vy{{\bm{y}}}

\DeclareMathAlphabet{\mathsfit}{\encodingdefault}{\sfdefault}{m}{sl}
\SetMathAlphabet{\mathsfit}{bold}{\encodingdefault}{\sfdefault}{bx}{n}

\newcommand{\E}{\mathbb{E}}

\newcommand{\R}{\mathbb{R}}

\newcommand{\KL}[2]{\mathrm{KL}\left(#1 \big\| #2\right)}

\makeatletter
\def\thickhline{%
  \noalign{\ifnum0=`}\fi\hrule \@height \thickarrayrulewidth \futurelet
   \reserved@a\@xthickhline}
\def\@xthickhline{\ifx\reserved@a\thickhline
               \vskip\doublerulesep
               \vskip-\thickarrayrulewidth
             \fi
      \ifnum0=`{\fi}}
\makeatother

\newlength{\thickarrayrulewidth}
\newtheorem{lemma}{Lemma}[section]
\newtheorem{theorem}[lemma]{Theorem}
\newtheorem{corollary}[lemma]{Corollary}

\newtheorem{remark}{Remark}
\newtheorem{assumption}{Assumption}

\newcommand*\samethanks[1][\value{footnote}]{\footnotemark[#1]}

\title{Multi-step Consistency Models: Fast Generation with  \\ Theoretical Guarantees}

\author[$\dagger$]{\normalsize Nishant Jain\thanks{Equal Contribution, Mail to \href{nj27@illinois.edu}{nj27@illinois.edu}, \href{xhuangck@connect.ust.hk}{xuh031@ucsd.edu} }}
\author[$\S$]{Xunpeng Huang\samethanks}
\author[$\S$]{Yian Ma}
\author[$\dagger$]{Tong Zhang}

\affil[$\dagger$]{University of Illinois Urbana-Champaign}
\affil[$\S$]{University of California San Diego}

\begin{document}

\date{}
\maketitle

\begin{abstract}
     Consistency models have recently emerged as a compelling alternative to traditional SDE-based diffusion models. They offer a significant acceleration in generation by producing high-quality samples in very few steps. Despite their empirical success, a proper theoretic justification for their speed-up is still lacking. 
 In this work, we address the gap by providing a theoretical analysis of consistency models capable of mapping inputs at a given time to arbitrary points along the reverse trajectory. We show that one can achieve a KL divergence of order \( O(\varepsilon^2) \) using only \( O\left(\log\left(d/\varepsilon\right)\right) \) iterations with a constant step size. 
 Additionally, under minimal assumptions on the data distribution (non-smooth case)—an increasingly common setting in recent diffusion model analyses—we show that a similar KL convergence guarantee can be obtained, with the number of steps scaling as \( O\left(d \log\left(d/\varepsilon\right)\right) \). 
 Going further, we also provide a theoretical analysis for estimation of such consistency models, concluding that accurate learning is feasible using small discretization steps, both in smooth and non-smooth settings. Notably, our results for the non-smooth case yield best-in-class convergence rates compared to existing SDE/ODE-based analyses under minimal assumptions.

\end{abstract}

\section{Introduction}
Lately, diffusion models \cite{sohl2015deep} have become an important topic of research in computer vision and generative modelling \cite{song2019generative, croitoru2023diffusion, lugmayr2022repaint, song2020denoising, nichol2021glide, song2021maximum, ho2020denoising}, with applications ranging from generating images or videos \cite{epstein2023diffusion, chen2023control} controllably to other areas like drug/protein design \cite{gruver2023protein, guo2024diffusion}. These models comprise a forward process that gradually adds noise to the data, and then the generation is done by the corresponding denoising process, which is sometimes referred to as the reverse/generation process.
These forward/reverse processes can either be seen as transition kernels \cite{huang2024reverse, song2020denoising, ho2020denoising}\ or instead modeled as stochastic differential equations (SDEs) \cite{song2020score, song2020sliced, song2021maximum, chen2022sampling}. 
This is popularly referred as \textit{score based generative modeling} due to its reliance on the neural-network parameterized \textit{score function}, which at a given time instant is the gradient of log probability of the marginal distribution corresponding that time.

Following from the particle-density transition property,~\cite{song2020score} firstly argued that score-based diffusion method can be implemented by an ordinary differential equation (ODE) version.
That means, under the unbiased score estimation, the ODE-based generation will share the same marginal distribution of the particles as those in the SDE-based generations.
From a practical perspective, ODE-based methods can often bring the underlying distribution of particles closer to the data distribution faster than SDE-based methods while maintaining comparable generation quality~\cite{lu2022dpm, zhou2024fast, zhang2022fast}.
Moreover, researchers appreciate the deterministic update property in ODE-based methods since all the randomness is left in the particle initialization, which inspires the proposal of consistency models~\cite{song2023consistency}.

Motivated by distilling the deterministic mapping from ODE-based diffusion models,
the original consistency model paper \cite{song2023consistency} takes any point along the probability flow ODE to the start \textit{i.e.} the true data distribution with single-step generation function.
Following the attention growth, a recent work \cite{kim2023consistency} attempted to improve the consistency model by extending the consistency function to any timestamp pair along the reverse ODE trajectory and consequently proposed a multi-step training scheme to achieve this, showing empirical effectiveness. 
From a theoretical perspective, \cite{dou2024theory, li2024towards} investigates the sample efficiency or the number of iterations to train such consistency models.
However, the essential advantage of the consistency model in the inference process remains unknown. 
Although~\cite{lyu2024sampling} provides some initial exploration for the inference efficiency of the consistency model, for achieving the convergence, an $\tilde{O}(\varepsilon^2)$ step size is required that shares the same order as that in typical SDE or ODE-based inference algorithm.
That result can neither show the advantages of introducing the deterministic update distillation nor match the real practice experience. 
Therefore, a natural question is raised:
\begin{quote}
    \emph{Can the consistency model achieve convergence with a larger step size, matching the real practice experience, and what kind of convergence will it achieve?}
\end{quote}

In this work, we argue that the inference of consistency models via an adapted version of the multi-step updates allows a constant-level step size, which leads to a linear KL convergence toward the original data distribution under minimal smooth assumptions.
Specifically, {with this setup, we show that at the inference time, one can achieve $O(\varepsilon^2)$ error with a constant step size. We provide this analysis for two scenarios: a) having popular assumptions used in the diffusion model analysis \cite{chen2022sampling,song2020score,xu2024provably,lyu2024sampling}, which includes Lipschitzness of the score function, small score estimation error, finite second moment, and b) without assuming that the score function is Lipschitz \cite{chen2023improved, benton2024nearly}  a scenario that is recently considered and deemed closer to the to real-world applications.} We are able to achieve this convergence by using the multi-step generation where after every application of the consistency model, there is a noising step during inference. Intuitively, this noise is effective in cancelling the accumulative score approximation error. Along, with this, another major ingredient is the modified formulation of the original consistency model that can map a sample from a given time instant to any arbitrary instant along the reverse ODE. We thereby analyse a modified multi-step sampling (version adapted to this formulation) in the KL divergence. 

\noindent
 We summarize the major contributions of this work as follows:

\begin{itemize}
    \item We provide an inference time analysis to achieve the $O(\varepsilon^2)$ KL divergence in $O\left(\log(\frac{d}{\varepsilon})\right)$ and a constant step size, utilizing the consistency function corresponding to the reverse probability flow ODE.
    \item We further relax the smoothness assumption and provide the first analysis for consistency models under this scenario to adapt them more general data distributions, showing that the number of steps scales linearly in dimension as $O\left(d\log(\frac{d}{\varepsilon})\right)$. 
    \item We finally provide a theoretical analysis for estimating such consistency functions (under both smooth and non-smooth scenarios) and conclude that under fine-grained discretization at the training time, they can be estimated with very high accuracy. 
\end{itemize}

\subsection{Related Work}

\paragraph{SDE-Based Analysis of Diffusion Models.}
The foundational work establishing the effectiveness of diffusion models for generative tasks is the Denoising Diffusion Probabilistic Models (DDPM) framework introduced by \cite{ho2020denoising}. Building on this, \cite{song2020score} demonstrated that the forward noising process in DDPMs can be interpreted as a stochastic differential equation (SDE), laying the groundwork for continuous-time formulations of diffusion models. Subsequent works \cite{chen2022sampling,li2023towards, li2024towards, lee2022convergence} have focused on providing convergence guarantees for such SDE-based generative processes under smoothness or other regularity conditions.
More recent advancements have relaxed the smoothness assumptions traditionally imposed on the score function. For example, \cite{chen2023improved} and \cite{benton2023nearly} showed that the generative process can still converge to a Gaussian-perturbed version of the data distribution, even in the absence of score smoothness. Notably, the recent work of \cite{li2024d} achieved an improved convergence rate of $\mathcal{O}(d/T)$ for DDPM samplers without requiring smoothness of the score function.

\paragraph{ODE based diffusion analysis.} 
Since the discovery of the probability flow ODE, there has been growing interest in deterministic generation using diffusion models.
 One of the prominent works is DDIM \cite{song2020denoising}.
Others include a recent work \cite{chen2023probability} which showed under the standard assumptions the ODE also converges quickly.
Convergence analysis of this DDIM sampler has also been discussed in a couple of recent works \cite{li2024sharp, gao2024convergence, huang2025convergence, li2023towards, li2024unified} but require some additional assumptions. The best bounds for a general data distribution requires an assumption on the divergence of the estimated score \cite{li2024sharp}. A recent work \cite{li2024unified} instead exploited the Fokker-Planck equation but again with the additional assumption of Jacobi of estimated score for TV distance analysis. It also shows that without such additional assumptions the TV distance will always be lower bounded by a constant. 

\paragraph{Consistency Model Analysis} The original consistency models paper \cite{song2023consistency} proposed a single as well multi-step sampling scheme along with distillation based (which requries a pre-trained diffusion model to distill knowledge) and self-consistency training based setups. \cite{lyu2024sampling} provides theoritical analysis in the wassertian distance for both single and multi-step sampling, using the score estimation and lipschitz smoothness assumptions, along with the TV error analysis but with additional assumptions. It resulted in the step size/discretization complexity comparable to the state of the art SDE based diffusion. On the training side, a recent work showed how can we achieve consistency trajectory models \cite{kim2023consistency} where the consistency function can take you from any time $t_1$ to $t_2$ along the probability flow ODE. Another work \cite{daras2023consistent} exploits the consistency property of diffusion models to mitigate drifts in the data by modifying the de-noising score matching objective in diffusion. Furthermore, a recent work \cite{dou2024theory} also considered the analysis for consistency diffusion models from a statistical learning theory perspectives and proposed some statistical convergence rates for this based on the Wassertian distance. On similar lines, another work theoretically targeted the number of training steps required for consistency models \cite{li2024towards}.

\section{Preliminaries, Setup and Assumptions}
\label{subsec_3_notations}


We begin this section by discussion the formulation for typical SDE and ODE-based diffusion models. Next, we introduce the consistency model framework, a means to accelerate generation, and then provide the theoretical setup considered in this work. We then describe a multi-step generation algorithm under this formulation. Finally, we state the necessary assumptions for analyzing the convergence of these diffusion-based generation methods.

\paragraph{Diffusion Models.} Generative modelling via diffusion comprises of two parts. First corresponds to adding noise to the original data distrbution $p_{data}$ as a forward process which can be expressed as the following SDE:
\begin{align}
\label{general_sde_eq}
    d\vx_t = \mu(\vx,t) dt + \sigma(t)dw_t, \qquad x_0 \sim p_0 = p_{data} ,
\end{align}
where $\vx_t \in \mathbb{R}^d$, $t \in [0,T]$ where $T$ is the total time for which we run the noising forward process, $\mu, \sigma$ correspond to \textit{drift} and \textit{diffusion} coefficients and $w_t$ corresponds to the Brownian motion, $p_t= \text{law}(\vx_t)$ or the marginal distribution of the complete process at a given $t$. 
The corresponding backward probability flow ODE \cite{song2020score} would then be:
\begin{align}
\label{general_rode_eq}
    d\vx_t = \left[\mu(\vx_t,t) - \frac{1}{2}\sigma(t)^2\nabla\log p_t(\vx_t)\right]dt ,
\end{align}
where $\nabla \log p_t(\vx_t)$ is the score function.
It will have the same marginal as the SDE \cite{song2020score} and generation using it starts from $x_T \sim p_T$ in the reverse direction. Using the popular choice of OU process as the forward noising procedure for these diffusion models results in $\mu(\vx_t,t) = -\vx_t$, $\sigma(t)=\sqrt{2}$. Solving the SDE results in the following equation for the forward process:
\[
        \rvx_t = e^{-t} \rvx_{0} + \sqrt{1-e^{-2t}} z, \quad
        z \sim \mathcal{N}(0,I_d), \quad \rvx_0 \sim p_{data}
\]  
The marginal, joint, and conditional distribution w.r.t. $\rvx_t$ is denoted as
\begin{equation}
    \label{def:fwd_prob}
    \rvx_t\sim p_t,\quad (\rvx_{t^\prime}, \rvx_t)\sim p_{t^\prime, t},\quad \text{and}\quad p_{t|t^\prime}(\vx|\vx^\prime) = p_{t^\prime, t}(\vx^\prime, \vx)/p_{t^\prime}(\vx^\prime).
\end{equation}
A straightforward observation for this OU process then is that for the time period $0\le t^\prime < t\le T$, suppose $\rvx_{t^\prime}\sim p_{t^\prime}$ and 
    \[
        \rvx_t = e^{t'-t} \rvx_{t'} + \sqrt{1-e^{2(t'-t)}} z, \quad
        z \sim \mathcal{N}(0,I_d),
    \]    
where the underlying distribution of $\rvx_t$ is $p_t$.
The resultant probability flow ODE corresponding to this OU process becomes:
\begin{align}
\label{true_pf_ode}
  d \rvx_t = \left(-\rvx_t - \vs_t(\rvx_t)\right)dt ,
  \qquad
  \vs_t(\vx) = \nabla \log p_t(\vx) ,
\end{align}

\paragraph{Estimating the score function:} The true score function ($\vs_t$) is usually not available in the real world scenarios and is estimated
via \textit{denoising score matching} \cite{song2019generative}. Denoting the estimated score function as $\hat{\vs}_t(\cdot)$, it will result in the following probability flow ODE, which is also termed as \textit{empirical PF ODE} \cite{song2023consistency}:
\begin{align}
\label{empirical_pf_ode}
    d \hat{\rvx}_t = \left(-\hat{\rvx}_t - \hat{\vs}_t(\hat{\rvx}_t)\right)dt,
\end{align}
where $\hat{\rvx}_t$ can be treated as the empirical counterpart of $\rvx_t$ (and $\hat{p}_t$ as the counterpart for $p_t$, $\hat{\rvx}_t \sim \hat{p}_t$) which evolves according to estimated $\hat{\vs}_t$ as against the true score function.

\paragraph{Consistency Model.} For a given process $\{\vx_t\}_{t\in [\delta, T]}$ following the probability flow ODE (Eq.\ref{true_pf_ode}), \cite{song2023consistency} discussed the existence of a consistency function $f (x_t, t)$ as a backward mapping $f: \mathbb{R}^d \times \mathbb{R^+} \rightarrow \mathbb{R}^d$, which maps process at any time $t$ to the start of the trajectory $f(\vx_t,t) = \vx_\delta \, \forall \, t\in [\delta, T]$. Intuitively, this consistency function is associated to the velocity field $v: \mathbb{R}^d \times \mathbb{R^+} \rightarrow \mathbb{R}^d $ of the corresponding ODE: $d\vx_t = v(\vx_t,t)dt$.
The paper argued that estimating this function through the empirical PF-ODE (Eq.\ref{empirical_pf_ode}) can replace the iterative generation process in a single step and also proposed a distillation-based training scheme to achieve this.  

\noindent
We consider an alternative formulation for this consistency function which instead of always mapping to the start, can map to any arbitrary instant along the reverse ODE. To formalize this, we say, corresponding to the probability flow ODE in Eq. \ref{true_pf_ode}, there exists some consistency function $\vf\colon \R\times \R\times \R^d \rightarrow \R^d$ satisfying: 
\begin{equation*}
        \vf(t^\prime,t, \rvx_t)\sim p_{t^\prime}\quad \mathrm{when}\quad  \rvx_t\sim p_t.
\end{equation*}

Denoting the corresponding process associated with the empirical PF-ODE as $\{\hat{x}_t\}_{t\in [\delta,T]}$, similarly, we say that there exists a corresponding consistency function
, i.e.,
\begin{equation*}
    \hat{\vf}(t^\prime,t, \hat{\rvx}_t)\sim p_{t^\prime}\quad \mathrm{when}\quad \hat{\rvx}_t\sim \hat{p}_t.
\end{equation*} 
Since it might not seem obvious whether obtaining such a formulation for empirical PF-ODE is possible or not, we also provide a theoretical analysis for estimating this $\hat{f}$.

\noindent
\textbf{Notational Remark.} For the proofs provided in the appendix corresponding to the theorems mentioned in the main paper, we sometimes denote the variable corresponding to true (Eq. \ref{true_pf_ode}) and empirical PF ODE (Eq. \ref{empirical_pf_ode}) at $k^{th}$ point of a sequence of time stamps $t_k$ by $\rvx_{k}$, $\hat{\rvx}_k$ respectively as against using $\rvx_{t_k}$, $\hat{\rvx}_{t_k}$. \\

\subsection{Multi-step Sampling using consistency functions}
\label{subsec_3_multi_step_algo}

\noindent
We now consider a sequence $0< t_0 < t_1 < t_2 < \cdots < t_K$ and let $t_j' \in
[0, t_j)$. Also, from the notations defined above, the law at time $t_k$ for process $\hat{\rvx}_t$ is denoted as $\hat{p}_{t_K}$ which can also be seen as an approximation of $p_{t_K}$ (corresponding to $\rvx_{t_k}$). Similarly, to keep consistency from the notation above, we will denote the joint of distribution for $(\hat{\rvx}_{t_0}, \hat{\rvx}_{t_1}, ..., \hat{\rvx}_{t_K})$ (and correspondingly $({\rvx}_{t_0}, {\rvx}_{t_1}, ..., {\rvx}_{t_K})$) as $\hat{p}_{t_0, t_1,..., t_K}$ (${p}_{t_0, t_1,..., t_K}$ respectively). 
The multi-step sampling using this empirical PF ODE is defined in Algorithm \ref{alg:1}. It can be interpreted as first following the empirical PF ODE (eq. \ref{empirical_pf_ode}) to go from time $t_k$ to some $t'_{k-1}$ in the reverse (generation) direction and then take a step away (forward) from the generation by adding noise. Figure \ref{fig_alg} shows both these steps. This noise can act as a regularizer (smoothener) for the generation process and translates the ODE based generation from the consistency model to some intermediate between ODE and SDE based generation. This, as we will discuss below, leads to better convergence guarantees.
Along with this Algorithm \ref{alg:1}, we also define sampling when using the true values in this algorithm (Algorithm \ref{alg:2} provided in the Appendix) which leads to the true data distribution. 



\vspace{0.14in}

\noindent
We will now consider the convergence of this multi-step sampling Algorithm \ref{alg:1} in the KL divergence w.r.t. true data distribution (or equivalently Algorithm 2) in the subsection below. For clarity, we also define the notation corresponding to the step size corresponding to the sequences defined above, as $h_k = t_k-t_{k-1}$ and $h'_k = t_k-t'_{k-1}$. These can be interpreted as step sizes corresponding to travelling \textit{along the reverse trajectory} and \textit{going back along the forward} respectively. Thus, we have the following set of relations:
\begin{align}
\label{h_h'_dependency}
    h_k=t_k-t_{k-1}<h'_k = t_k-t'_{k-1}
\end{align}
since based on our definition of the sequence $t'_k$ above, we will have $t'_{k-1} < t_{k-1}$.

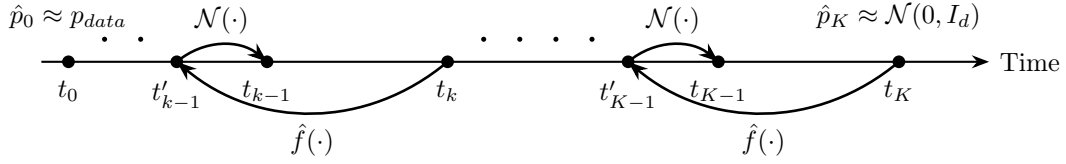
\begin{figure}
\centering
\begin{tikzpicture}[scale=1.2, >=Stealth, 
    thick, 
    every node/.style={font=\small}
  ]

  \draw[->] (-0.5,0) -- (10,0) node[right] {Time};

  \node at (-0.2,0.5) {$\hat{p}_0 \approx p_{data}$};
  \fill (-0.2,0) circle (2pt) node[below=3pt] {$t_{0}$};
  \fill (0.2,0.25) circle (0.8pt) node[below=3pt] {};
  \fill (0.6,0.25) circle (0.8pt) node[below=3pt] {};
  \fill (1,0) circle (2pt) node[below=3pt] {$t'_{k-1}$};
  \fill (2,0) circle (2pt) node[below=3pt] {$t_{k-1}$};
  \fill (4,0) circle (2pt) node[below=3pt] {$t_k$};
  \fill (4.4,0.25) circle (0.8pt) node[below=3pt] {};
  \fill (4.8,0.25) circle (0.8pt) node[below=3pt] {};
  \fill (5.2,0.25) circle (0.8pt) node[below=3pt] {};
  \fill (5.6,0.25) circle (0.8pt) node[below=3pt] {};
  \fill (6,0) circle (2pt) node[below=3pt] {$t'_{K-1}$};
  \fill (7,0) circle (2pt) node[below=3pt] {$t_{K-1}$};
  \fill (9,0) circle (2pt) node[below=3pt] {$t_K$};
  \node at (9,0.5) {$\hat{p}_K \approx \mathcal{N}(0,I_d)$};

  \draw[->, bend left=40, line width=1pt] (4,0) to node[below] {$\hat{f}(\cdot)$} (1,0); 
  \draw[->, bend left=40, line width=1pt] (1,0) to node[above] {$\mathcal{N}(\cdot)$ } (2,0);
  \draw[->, bend left=40, line width=1pt] (9,0) to node[below] {$\hat{f}(\cdot)$} (6,0); 
  \draw[->, bend left=40, line width=1pt] (6,0) to node[above] {$\mathcal{N}(\cdot )$ } (7,0);

\end{tikzpicture}
\caption{Demonstrating the step 3 ($\hat{f}$) and 5 ($\mathcal{N}$) of our algorithm w.r.t. the reverse ODE.} 
\label{fig_alg}

\end{figure}

As discussed in the introduction, we consider two analysis based on the assumptions used. In both of our analysis the following assumptiosn are common:

\begin{assumption}
\label{assumption_score_est}
    \textit{The score function estimate} $\{\hat{\vs}_t\}_{1\leq t \leq T}$ obeys for all $t$:
    \begin{equation}
    \label{eq_score_est}
     \mathbb{E}_{\rvx \sim p_{t}}\left [\|\hat{\vs}_{t}(\rvx)- \vs_{t}(\rvx)\|^2 \right ] \leq \varepsilon^2_{score}.
\end{equation}

\end{assumption}

\begin{assumption}
\label{finite_moment}
    The data distribution $p_{data}$ has finite second order moment $\mathbb{E}_{\rvx_0 \sim p_{\text{data}}} \left[ \|\rvx_0\|_2^2 \right] = m_2 < \infty$.
\end{assumption}

\noindent
 Both of these assumptions are pretty standard and have been used in all of the prior works in theoretical analysis of diffusion based generation. We now formalize our setup and discuss the multi-step sampling scheme using consistency models.

\noindent
\begin{algorithm}[H]
\caption{Multi-Step Consistency Generation}
\label{alg:1}
\begin{algorithmic}[1]
\STATE Sample $\hat{\rvx}_K \sim \hat{p}_{t_K}$
\FOR{$k = K, K-1, \ldots, 1$}
    \STATE $\hat{\rvx}'_{k-1} = \hat{f}(t'_{k-1}, t_k, \hat{\rvx}_K)$ \\
    \STATE Sample $z \sim \mathcal{N}(0, I_d)$
    \STATE $\hat{\rvx}_{k-1} = e^{t'_{k-1} - t_{k-1}} \hat{\rvx}'_{k-1} + \sqrt{1 - e^{2(t'_{k-1} - t_{k-1})}}\, z$ 
\ENDFOR
\STATE \textbf{Output} $\hat{\rvx}_0$
\STATE $\hat{p}_{t_0}$ denotes the density of $\hat{\rvx}_0$
\STATE $\hat{p}_{t_0,.., t_K}$ denotes joint density of $(\hat{\rvx}_0, \hat{\rvx}_1, \ldots, \hat{\rvx}_K)$
\end{algorithmic}
\end{algorithm}
\section{Main Results}
\label{sec_3_results}
\noindent
In this section, we provide the theoretical results which advocate for the empirical effectiveness \cite{heek2024multistep} (generating high quality samples in few steps) of the consistency model based formulation. We first provide inference time analysis utilizing the exact consistency function for the empirical PF-ODE (Eq.\ref{empirical_pf_ode}). Then, we provide a theoretical analysis on such consistency functions can be accurately estimated using the consistency distillation training scheme \cite{kim2023consistency}. 
This segregation of training and inference time is done since the usual applications of these diffusion models are majorly concerned with accurate estimation during train time and once trained, are efficient in generation (or inference). Thus, we can train them with arbitrarily small discretization for many steps but for inference only a few steps (and consequently a high discretization complexity) are required for high quality generation.
We now discuss the convergence analysis for the multi-step sampling in Algorithm \ref{alg:1}.

\subsection{Convergence in the KL divergence for multi-step sampling}
\label{subsec_3_analysis}

\noindent
We first discuss the analysis involving the smoothness of the score function, which has also been widely used in the literature \cite{lyu2024sampling, xu2024provably, chen2022sampling}. The assumption is as follows: 
\begin{assumption}
\label{smoothness_assumption}
    The approximate score function $\hat{\vs}_t$ $L-$Lipschitz with $L\geq 1$ for all $t \geq 1$.
\end{assumption}

\noindent
Notice, that it is a bit different from the previous works \cite{chen2022sampling, xu2024provably, lyu2024sampling, chen2023improved} since they assume the true score to be Lipschitz.
We provide another analysis where we relax this assumption and incur additional dependence on the dimension $d$ replacing $L$. We now provide our first main result as follows.


\begin{theorem}
\label{thm:smooth_score_case}
For  Algorithm~\ref{alg:1}, if $h'_k \leq \frac{1}{2(1+L)}$, along with assumptions \ref{assumption_score_est}, \ref{finite_moment}, \ref{smoothness_assumption}  provided, we have:
  \begin{align}
  \label{smooth_bound_complete_eq}
    KL(p_{t_0}\|\hat{p}_{t_0}) \leq (d+m_2) e^{-T} +  e^{2}{h'}_k^2 \varepsilon^2_{score} \sum_{k=1}^K\frac{1}{4(t_k-t'_k)}
\end{align}

\end{theorem}
\textit{Proof Sketch.}  Please refer Appendix~\ref{appendix:proof_of_thm1} for the complete proof. Here, we discuss a higher level sketch. The proof involves considering the two sources of error: a) Initialization error due to starting the reverse process from a normal distribution (lemma \ref{convergence_of_ou}) and b) the error incurred by using empirical PF ODE (eq. \ref{empirical_pf_ode}) instead of the true PF ODE (eq. \ref{true_pf_ode}, which will depend on $\varepsilon_{score}$). Also, for intuition, the tight control of the KL is due to adding the noise and then re-applying the consistency function $\hat{f}$ in the Algorithm \ref{alg:1} (steps 3 and 5) of the empirical PF-ODE (Eq.\ref{empirical_pf_ode}). 
    

\noindent
Since we know that the total time of the forward process $T$ would be $\sum^K_{k=1}h_k$, we can conclude the following from this theorem:
\begin{corollary}
 Under Assumptions \ref{assumption_score_est}, \ref{finite_moment}, \ref{smoothness_assumption},   Algorithm \ref{alg:1} achieves the KL divergence error $O(\varepsilon^2)$, if we run it for a total time $T = \log(\frac{d+m_2}{\varepsilon})$
    with the constant step size say $h_k = \frac{1}{3(L+1)}$ and $h'_k = \frac{1}{2(L+1)}$ (which leads to $t_k-t'_k = h'_{k+1}-h_{k+1} = \frac{1}{6(L+1)}$), thereby inducing an iteration/discretization complexity $K=\frac{T}{h_k} = 3(L+1)\log(\frac{d+m_2}{\varepsilon})$ given that the score estimation error from denoising score matching is $\varepsilon_{score} = O\left(\frac{\varepsilon}{\sqrt{log(\frac{d+m_2}{\varepsilon})}}\right)$. 
    
\end{corollary}

\noindent
Therefore, we can have the step size (and correspondingly the number of iterations) independent (logarithmically dependent) of $\varepsilon$ to achieve $O(\varepsilon^2)$ accuracy in the KL divergence which is better then any of the existing results $\left(O(\frac{1}{\epsilon})\right)$ for DDPM \cite{ho2020denoising}/DDIM \cite{song2020denoising} samplers which need step size to be atleast $O(\varepsilon)$ for the SDE based generation \cite{li2024d} (thereby inducing an iteration complexity of $O(\frac{1}{\varepsilon})$) and also require much stricter assumptions for ODE based generation (like the error between Jacobi of true and estimated score is small) \cite{li2024sharp,li2024unified}.
This shows the effectiveness of using the consistency model based formulation for generation tasks with constant step sizes. 
Also, we can see that this would not be possible to achieve without multi-step sampling, which requires adding noise at each step into the generated samples. This acts as a regularizer and helps to prevent error accumulation. Intuitively, this is similar to what a stochastic differential equation (SDE) achieves in the score-based formulation. Therefore, it is reasonable to conclude that the theoretical advantages of using consistency models become effective when employing the multi-step iterative sampling approach, while requiring far fewer steps compared to standard diffusion-based generation.

\noindent
We now consider relaxing the assumption 3.3 and provide the convergence in KL for the multi-step sampling in the next subsection.

\subsection{The Non-Smooth Case}
\label{non_smooth_subsec}
\noindent
Taking inspiration from recent works \cite{chen2023improved,benton2024nearly} on relaxing the smoothness assumption of the true score function, we first define the noise schedule/conditional variance for $\rvx_t$ given $\rvx_0$ for the forward OU process as $\sigma_t = 1-e^{-2t}$ and arrive at the following result for the multi-step sampling (Algorithm \ref{alg:1}) in absence of any smoothness. 
\begin{theorem}
\label{thm:non_smooth_3.3}
   For Algorithm \ref{alg:1}, using only assumptions \ref{assumption_score_est} and \ref{finite_moment}, if we have $ {h'}_k < \frac{\sigma^2_{t'_{k-1}}}{d}$ and the number of iterations $K=\frac{d}{\sigma^2_{t'_{k-1}}}\log\left(\frac{(d+m_2)}{\varepsilon_{score}}\right)$, then:
  \begin{align}
KL(p_{t_\delta}\|\hat{p}_{t_\delta}) \leq (d+m_2) e^{-T} +  e^4{h'}_k^2\varepsilon^2_{score} \sum_{k=1}^K\frac{1}{4{(t_k-{t'}_k)}}
  \end{align}
where $t'_0 = \delta > 0$.
\end{theorem}

\begin{remark}
The proof can be found in the Appendix \ref{appendix_non_smooth}. Again similar to the idea of Theorem \ref{thm:smooth_score_case} proof, we have to consider both initialization error and the error due to empirical ODE (eq. \ref{empirical_pf_ode}). However, since the score function hasn't been provided as smooth here, bounding the error due to the empirical ODE will be a bit more tricky here. Taking inspiration from the previous works, we first try to bound the operator norm of the score function along the true trajectory since, using the fact that the forward process is just a convolution with gaussian distribution and the perturbation can be bounded. 
\end{remark}

\noindent
It is easy to observe that absence of smoothness (the constant $L$) induces a factor of $d$ but the remaining result is similar to the previous theorem and thus, we can again have a similar conclusion as follows.

\begin{corollary}
Under Assumptions \ref{assumption_score_est}, \ref{finite_moment}, Algorithm \ref{alg:1} achieves the KL divergence error $O(\epsilon^2)$, if we run it for a total time $T= \log(\frac{d+m_2}{\epsilon})$
    with the constant step size, say, $h_k = \frac{1-e^{-\delta}}{2d}$ and $h'_k = \frac{1-e^{-\delta}}{d}$, thereby inducing an iteration/discretization complexity $K = \frac{T}{h_k} = \frac{2d}{1-e^{-\delta}}\log(\frac{d+m_2}{\epsilon})$ given that the score estimation error from denoising score matching is $\varepsilon_{score} = O\left(\frac{\epsilon}{\sqrt{\log(\frac{d+m_2}{\epsilon})}}\right)$, better then any of the existing results \cite{benton2023nearly} $\left(O\left(\frac{d}{\epsilon^2}\right)\right)$ for DDPM \cite{ho2020denoising} sampler when the smoothness of the score function is not assumed.
    
\end{corollary}

\noindent
Both this theorem and the previous theorem \ref{thm:smooth_score_case} suffer from the limitation of $h'_k$ being bounded. This limits their applicability to the original consistency model formulation \cite{song2023consistency} which always takes as the end of the reverse flow ODE (or the data distribution)and thus for that the sequence $t'_k$ should be set to 0. This issue has further been discussed in the appendix after the proof of Lemma \ref{main_lemma_smooth}.\\
Seeing the proof of both these theorems, it can be observed that the analysis is almost tight and thus, to resolve this limitation, exploring other metrics like TV distance might be an interesting direction. 
 We now discuss learning the consistency function formulation $\hat{f}$ corresponding to the empirical ODE using the distillation technique proposed in the consistency model paper \cite{song2023consistency}.



\subsection{Consistency Distillation Training to estimate $\hat{f}$}
\label{learning_subsec}

\noindent
In the previous subsection, we discussed how we can exploit the given formulation of the consistency function \textit{i.e.} $\hat{f}(t',t,\rvx_t)$ corresponding to the empirical PF ODE to achieve state of the art convergence results when doing iterative multi-step sampling. However, it is still not clear whether such a consistency function correponding to empirical ODE (eq. \ref{empirical_pf_ode}) can be learned efficiently or not. In this section, we discuss this learning of such consistency function. \\

\noindent
The original consistency model paper proposes two schemes to learn any consistency function: distillation based training and the self-consistency training. Here, we will consider the first case. It involves using an ODE solver $\Phi$ and distilling its knowledge into the consistency model. Let us denote the parameterized approximation of $\hat{f}$ as $\hat{f}_\theta$. The distillation based training involves considering the true process at $t_{k+1}$: $\rvx_{t_{k+1}} = e^{-t_{k+1}}\rvx_0 + \sqrt{1-e^{-2t_{k+1}}}\epsilon,$ $\epsilon\sim \mathcal{N}(0,I_d)$ and taking one step back to get $\hat{\vx}^\phi_{t_k}$ using a pretrained diffusion model as ODE solver $\phi$ and the empirical PF ODE, denoting this overall one step update as $\Phi(\cdot)$ : 
\begin{align}
\label{update_x}
\hat{\vx}^\phi_{t_k} = \Phi(\rvx_{t_{k+1}},t_{k+1}, t_k) = \rvx_{t_{k+1}} - ({t_{k+1}-t_k}) \hat{s}_{t_{k+1}}(\rvx_{t_{k+1}})
\end{align}
The objective $\mathcal{L}_{CD}$ then is to feed both of these to $\hat{f}_\theta$ and minimize the euclidean distance between the resulting outputs:
\begin{align}
\label{cd_loss}
\mathcal{L}_{\text{CD}}(\theta, \theta^{-}; \Phi) 
:= \mathbb{E}\left[ \lambda(t_n) \left\| \hat{f}_{\theta}(t_0, t_{n+1}, \rvx_{t_{n+1}}) - \hat{f}_{\theta^{-}}(t_0, t_{n}, \hat{\vx}^{\Phi}_{t_n}) \right\|_2^2 \right]
\end{align}
where $\theta^-$is just the running averages of the parameters, done for a stable training and also for faster convergence and $\lambda(\cdot) \in \mathbb{R}^+$ is just a positive weighing function \cite{song2023consistency}. Now, to analyze the difference between true consistency function and the learned consistency function $\hat{f}_\theta$ via the above objective, we first state some assumptions on the training as well as on the parametrized function $\hat{f}_\theta$ itself. These are again standard in literature and have been used the all recent works involving consistency model analysis \cite{kim2023consistency, song2023consistency, lyu2024sampling}.
\\

\noindent
\begin{assumption}
\label{cd_assumption}
    We have the following assumption \cite{song2023consistency, lyu2024sampling} on consistency distillation error for the approximator of $\hat{f}$ \textit{i.e.} $\hat{f}_\theta$:
\begin{equation}
    \mathbb{E}_{\rvx_{t_{k+1}} \sim p_{t_{k+1}}}\left[\|\hat{f}_\theta(t'_{k-1}, t_{k+1}, \rvx_{t_{k+1}}) - \hat{f}_\theta(t'_{k-1}, t_k, \hat{\vx}^\phi_{t_k})\|_2^2\right] 
    \leq \varepsilon^2_\text{cd}(t_{k+1} - t_k)^2, \, \forall \, k \in [1, K - 1],
\end{equation}
\end{assumption}

\begin{assumption}
\label{smooth_f}
    \textit{$\hat{f}_\theta$ is $L_f-$lipschitz} \cite{kim2023consistency, song2023consistency, lyu2024sampling}. \\
\end{assumption}

\noindent
\textbf{Verifying Assumption \ref{smooth_f}.}  Since we have assumed that $\hat{\vs}_t$ is smooth in one of the analysis above, it is straightforward to verify that $\hat{f}(t',t,\cdot)$ would satisfy the following (for intution consider error accumulated in naive euler discretization):
\[
\hat{f}(t', t, x) = 1 + (t - t') + {O}((t - t')^2) \cdot x + \left((t - t') + {O}((t - t')^2)\right) \cdot s_t(x)
\]
Abstracting out the higher order $(t-t')$ terms since it is small, we will have:
\begin{align}
\label{approx_smooth_hat_f}
\|\hat{f}(t',t,x)-\hat{f}(t',t,y)\|_2
&\leq \|\left(1+ O(t-t')\right) (x-y)
+ O(t-t') (\hat{\vs}_t(x)-\hat{\vs}_t(y)) \|_2 \notag \\
&\leq \left(1+ O(t-t')\right) \|(x-y) \|_2
+ O(t-t') \|(\hat{\vs}_t(x)-\hat{\vs}_t(y)) \|_2 \notag\\
&\leq
\left(1+(1+L)\cdot O(t-t')\right) \|x-y\|_2  
\end{align}
where $L$ is the Lipschitz of $\hat{\vs}_t$ (Assumption \ref{smoothness_assumption}). For a tight upper bound, we can have:
\[
\|\hat{f}(t',t,x)-\hat{f}(t',t,y)\| \leq 
e^{(1+L)(t-t')} \|x-y\|  
\]
Thus, assuming that $\hat{f}$ would be lipschitz smooth is a reasonable assumption and $L_f$ would be approximately same as $(1+(1+L)(t-t'))$.
Also, we can now directly extend this finding to $\hat{f}_\theta(t',t,x)$ since it should just incur some additional error related to $\varepsilon_{cd}$ which would be of other order $O((t-t')\varepsilon_{cd})$. This will lead to the following:
\[
\| \hat{f}_{\theta}(t'_{n-1}, t_{n}, \vx_{t_n}) - \hat{f}_{\theta}(t'_{n-1}, t_{n}, \vy_{t_n}) \|_2 \leq (t_n - t'_{n-1})(\varepsilon_{\text{cd}}) + L_f \| \vx_{t_n} - \vy_{t_n} \|_2.
\]
Therefore, we can assume that $\hat{f}_\theta(t',t,\cdot)$ would be $L_f$-lipschitz.
\noindent
We now provide the following theorem regarding the difference between the estimated consistency function using the distillation based training and true consistency function corresponding to the PF-ODE using the above assumptions.


\begin{theorem}
\label{thm:gap_f_hat_f_theta}
\textbf{(Bounding error between $\hat{f}$ and estimated $\hat{f}_\theta$)}. Following the definition of $\hat{f}$ for some discretization $\{t_n\}_{n\in [1,N]}$ for the consistency distillation training, 
under assumption 3.1-3.5, we have:
\begin{align*}
\mathbb{E} \|\hat{f}_\theta(t'_{n-1}, t_{n}, \rvx_n) - \hat{f}(t'_{n-1}, t_{n}, \rvx_n)\|^2_2 \leq  L_fe^{{h_{n-1}}/{2}}(L^{{3}/{2}}d^{{1}/{2}}h_{n-1}) + \varepsilon_{cd}(t_n-t_1).
\end{align*}
\end{theorem}
\begin{proof}
    For proof, please refer to Appendix~\ref{appendix:Thm3_6}.
\end{proof}
\noindent
\textbf{Achieving a good approximation of $\hat{f}$}. Based, on the previous theorem, it is easy to observe that the approximation mainly depends on the consistency distillation training error and the \textit{training time discretization}, both of which can be made arbitrarily small during training and thus, we can argue achieving a very close approximation for $\hat{f}$. Now, we will again consider this approximation analysis but this time when the smoothness assumption on score function is not provided.

\paragraph{Non-smooth case.} We will now consider the scenario where we do not have assumption \ref{smooth_f} where we have shown how to bound the expected value for $\|\nabla \vs_t(\cdot)\|$ on the true trajectory, since if $\hat{\vs}_t$ is not smooth, we cannot verify it and thus, is not a good assumption to have. Here, will use the idea from the proof of theorem \ref{thm:non_smooth_3.3} and will bound the expected difference instead. We can understand this for the true $f$ first as follows for some $\rvx, \rvy \in \mathbb{R}^d$:

\begin{align*}
&\mathbb{E}\|\hat{f}(t',t,\rvx)-\hat{f}(t',t,\rvy)\|_2\\
&\leq \mathbb{E}\|(1+ (t-t')) (\rvx-\rvy)
+ (t-t') (\vs_t(\rvx)-\vs_t(\rvy)) \|_2 \qquad \qquad \qquad\text{(similar argument as Eq. \ref{approx_smooth_hat_f})} \\
&\leq
(t-t') E\|\rvx-\rvy\|_2 + \frac{d}{\sigma_t^2}(t-t') E\left[\|\rvx-\rvy\|_2 \exp\left(\frac{\|\rvx-\rvy\|}{\sigma_t^2}\right)  ,\right] \qquad \text{(Lemma \ref{expected_smooth} Appendix)} \\
&\leq (t-t') E\|\rvx-\rvy\|_2 + \frac{2d}{\sigma_t^2}(t-t') E\|\rvx-\rvy\|_2 \qquad\qquad \qquad \qquad \quad \qquad \text{(when $\rvx,\rvy$ are close)}
\end{align*}

\begin{lemma}
\label{lemma_non_smooth_cf_lipschitz}
{(\textbf{Validating the assumption on $\hat{f}_\theta$ in the non-smooth case.})
Using the exponential integrator while the consistency distillation training:
\[
\hat{\vx}^\phi_{t_n} = e^{t_{n+1}-t_n} \rvx_{t_{n+1}} + (e^{t_{n+1}-t_n} - 1) \hat{\vs}_{t_{n+1}}(\rvx_{t_{n+1}}),
\]
and given the assumptions \ref{assumption_score_est}, \ref{finite_moment}, \ref{cd_assumption},  
we have:
\begin{align*}
\mathbb{E}\|f_\theta(t_1, t_n, \rvx_{t_n}) - f_\theta(t_1, t_n, \rvy_{t_n})\|_2 
& \leq 2(t_n-t_1)\varepsilon_{cd} + 2\varepsilon_{score}(t_n-t_1)  + n\mathbb{E}\|\rvx_{t_n}-\rvy_{t_n}\|_2 
\end{align*}
}
\noindent
where again $\vy_{t_n}$ lie on a (correspond to) different probability flow ODEs (for a given time-stamp $t_n$).
\end{lemma}
\textit{Proof Sketch.} A rough sketch starting from $\hat{\rvx}_{t_n}={\rvx}_{t_n}$ (similarly for ${\rvy}$) and decomposing the terms corresponding to $\rvx, \rvy$ into the additional error aggregation when ${\hat{\rvx}}_{t_{i}}$ is mapped to $t_1$ as against $\hat{\rvx}_{t_{i-1}}$ and similarly for $\rvy$, thereby bounding their difference using the sum of these terms.
Also, for this non-smooth scenario, we bound the expectation using our lemma \ref{expected_smooth} by bounding the expected hessian (or gradient of score). Please refer Appendix \ref{appendix:Thm3_6} for the complete proof. It is straightforward to further adapt for any $t'_{n-1}$ as the first parameter as against $t_1$. It has been omitted here for simplicity.

\noindent
Now, we again provide the analysis for the approximated consistency model (counterpart of Theorem \ref{thm:gap_f_hat_f_theta}) for the non-smooth case:
\begin{theorem}
\label{3.8thm_gap_f_non_smooth}
Following the definition of $\hat{f}$ for some discretization $\{t_n\}_{n\in [1,N]}$ in the consistency distillation training, 
using assumptions \ref{assumption_score_est}, \ref{finite_moment}, \ref{cd_assumption}, we have:
\begin{align*}
\E  
  \|
  f(t_{n-1}',t_n,{\rvx}_n)-\hat{f}_\theta(t_{n-1}',t_n,{\rvx}_n)\|_2 \leq ne^{{h_{n-1}}/{2}}(L^{{3}/{2}}d^{{1}/{2}}h_{n-1}) + (t_n-t_1)\left(3\varepsilon_{cd} +2\varepsilon_{score} \right)
\end{align*}

\end{theorem}
   \textit{Proof Sketch.} The proof is similar as Theorem \ref{thm:gap_f_hat_f_theta} but here we instead utilize the Lemma \ref{lemma_non_smooth_cf_lipschitz} and bound the expectation of the term. It incurs a factor of $d$ which arises from the bound on the hessian. Also, here as against Theorem \ref{thm:gap_f_hat_f_theta}, we have bounded the error w.r.t. the true consistency function corresponding to the actual reverse ODE and thus, we incur the additional term involving the score estimation.
    The detailed proof is provided in Appendix \ref{appendix:Thm3_6}. 
\section{Conclusions}
\noindent
In this work, we provided a theoretical analysis for multi-step generation using consistency models, showing that the number of iterations $K$ (and consequently the step size) for the probability flow ODE based generation  can be independent of $\varepsilon$ to generate samples from a distribution which is $\varepsilon^2-$close in KL divergence to the target distribution, under minimal assumptions. Here, we have achieved $O(d\log(\frac{d}{\varepsilon}))$ convergence which is both \textit{state-of-the-art} w.r.t. to dimension $d$ and also w.r.t. $\varepsilon$ since it has a logarithmic dependence as against $O(\frac{1}{\varepsilon})$ in the current best ODE/SDE based samplers. We also don't require any strict assumptions on Divergence/Jacobi in our ODE based generation as the existing works and even relax the smoothness assumption. 
Furthermore, we also provide a theoretical analysis for estimating the formulation of the consistency function required for our analysis: which can take a point at a given time instant to arbitrary time instants along the reverse probability flow ODE. We show that this can be efficiently estimated using the distillation objective proposed in the original consistency models paper, given we use a fine discretization at the training time. Therefore, it can be concluded from this analysis that estimating accurate consistency function and combining them with iterative generation scheme involving noise addition can lead to much faster generation theoretically as against the typical DDPM typle samplers. An interesting future direction might be to further loosen the bound on step size to accomodate original consistency model formulation and maybe achieve tigher theoretical guarantees.

\bibliographystyle{apalike}
\bibliography{0_contents/ref}  

\begin{thebibliography}{}

\bibitem[Benton et~al., 2024]{benton2024nearly}
Benton, J., Bortoli, V.~D., Doucet, A., and Deligiannidis, G. (2024).
\newblock Nearly {\it d}-linear convergence bounds for diffusion models via stochastic localization.
\newblock In {\em The Twelfth International Conference on Learning Representations}.

\bibitem[Benton et~al., 2023]{benton2023nearly}
Benton, J., De~Bortoli, V., Doucet, A., and Deligiannidis, G. (2023).
\newblock Nearly $ d $-linear convergence bounds for diffusion models via stochastic localization.
\newblock {\em arXiv preprint arXiv:2308.03686}.

\bibitem[Chen et~al., 2023a]{chen2023improved}
Chen, H., Lee, H., and Lu, J. (2023a).
\newblock Improved analysis of score-based generative modeling: User-friendly bounds under minimal smoothness assumptions.
\newblock In {\em International Conference on Machine Learning}, pages 4735--4763. PMLR.

\bibitem[Chen et~al., 2023b]{chen2023probability}
Chen, S., Chewi, S., Lee, H., Li, Y., Lu, J., and Salim, A. (2023b).
\newblock The probability flow ode is provably fast.
\newblock {\em Advances in Neural Information Processing Systems}, 36:68552--68575.

\bibitem[Chen et~al., 2023c]{chen2022sampling}
Chen, S., Chewi, S., Li, J., Li, Y., Salim, A., and Zhang, A.~R. (2023c).
\newblock Sampling is as easy as learning the score: theory for diffusion models with minimal data assumptions.
\newblock In {\em International Conference on Learning Representations}.

\bibitem[Chen et~al., 2023d]{chen2023control}
Chen, W., Ji, Y., Wu, J., Wu, H., Xie, P., Li, J., Xia, X., Xiao, X., and Lin, L. (2023d).
\newblock Control-a-video: Controllable text-to-video generation with diffusion models.
\newblock {\em arXiv e-prints}, pages arXiv--2305.

\bibitem[Croitoru et~al., 2023]{croitoru2023diffusion}
Croitoru, F.-A., Hondru, V., Ionescu, R.~T., and Shah, M. (2023).
\newblock Diffusion models in vision: A survey.
\newblock {\em IEEE Transactions on Pattern Analysis and Machine Intelligence}, 45(9):10850--10869.

\bibitem[Daras et~al., 2023]{daras2023consistent}
Daras, G., Dagan, Y., Dimakis, A., and Daskalakis, C. (2023).
\newblock Consistent diffusion models: Mitigating sampling drift by learning to be consistent.
\newblock {\em Advances in Neural Information Processing Systems}, 36:42038--42063.

\bibitem[Dou et~al., 2024]{dou2024theory}
Dou, Z., Chen, M., Wang, M., and Yang, Z. (2024).
\newblock Theory of consistency diffusion models: Distribution estimation meets fast sampling.
\newblock In {\em Forty-first International Conference on Machine Learning}.

\bibitem[Epstein et~al., 2023]{epstein2023diffusion}
Epstein, D., Jabri, A., Poole, B., Efros, A., and Holynski, A. (2023).
\newblock Diffusion self-guidance for controllable image generation.
\newblock {\em Advances in Neural Information Processing Systems}, 36:16222--16239.

\bibitem[Gao and Zhu, 2024]{gao2024convergence}
Gao, X. and Zhu, L. (2024).
\newblock Convergence analysis for general probability flow odes of diffusion models in wasserstein distances.
\newblock {\em arXiv preprint arXiv:2401.17958}.

\bibitem[Gruver et~al., 2023]{gruver2023protein}
Gruver, N., Stanton, S., Frey, N., Rudner, T.~G., Hotzel, I., Lafrance-Vanasse, J., Rajpal, A., Cho, K., and Wilson, A.~G. (2023).
\newblock Protein design with guided discrete diffusion.
\newblock {\em Advances in neural information processing systems}, 36:12489--12517.

\bibitem[Guo et~al., 2024]{guo2024diffusion}
Guo, Z., Liu, J., Wang, Y., Chen, M., Wang, D., Xu, D., and Cheng, J. (2024).
\newblock Diffusion models in bioinformatics and computational biology.
\newblock {\em Nature reviews bioengineering}, 2(2):136--154.

\bibitem[Heek et~al., 2024]{heek2024multistep}
Heek, J., Hoogeboom, E., and Salimans, T. (2024).
\newblock Multistep consistency models.
\newblock {\em arXiv preprint arXiv:2403.06807}.

\bibitem[Ho et~al., 2020]{ho2020denoising}
Ho, J., Jain, A., and Abbeel, P. (2020).
\newblock Denoising diffusion probabilistic models.
\newblock {\em Advances in neural information processing systems}, 33:6840--6851.

\bibitem[Huang et~al., 2025]{huang2025convergence}
Huang, D.~Z., Huang, J., and Lin, Z. (2025).
\newblock Convergence analysis of probability flow ode for score-based generative models.
\newblock {\em IEEE Transactions on Information Theory}.

\bibitem[Huang et~al., 2024]{huang2024reverse}
Huang, X., Zou, D., Dong, H., Zhang, Y., Ma, Y.-A., and Zhang, T. (2024).
\newblock Reverse transition kernel: A flexible framework to accelerate diffusion inference.
\newblock {\em arXiv preprint arXiv:2405.16387}.

\bibitem[Kim et~al., 2023]{kim2023consistency}
Kim, D., Lai, C.-H., Liao, W.-H., Murata, N., Takida, Y., Uesaka, T., He, Y., Mitsufuji, Y., and Ermon, S. (2023).
\newblock Consistency trajectory models: Learning probability flow ode trajectory of diffusion.
\newblock {\em arXiv preprint arXiv:2310.02279}.

\bibitem[Lee et~al., 2022]{lee2022convergence}
Lee, H., Lu, J., and Tan, Y. (2022).
\newblock Convergence for score-based generative modeling with polynomial complexity.
\newblock {\em arXiv preprint arXiv:2206.06227}.

\bibitem[Li et~al., 2024a]{li2024towards}
Li, G., Huang, Z., and Wei, Y. (2024a).
\newblock Towards a mathematical theory for consistency training in diffusion models.
\newblock {\em arXiv preprint arXiv:2402.07802}.

\bibitem[Li et~al., 2023]{li2023towards}
Li, G., Wei, Y., Chen, Y., and Chi, Y. (2023).
\newblock Towards non-asymptotic convergence for diffusion-based generative models.
\newblock In {\em The Twelfth International Conference on Learning Representations}.

\bibitem[Li et~al., 2024b]{li2024sharp}
Li, G., Wei, Y., Chi, Y., and Chen, Y. (2024b).
\newblock A sharp convergence theory for the probability flow odes of diffusion models.
\newblock {\em arXiv preprint arXiv:2408.02320}.

\bibitem[Li and Yan, 2024]{li2024d}
Li, G. and Yan, Y. (2024).
\newblock $ o (d/t) $ convergence theory for diffusion probabilistic models under minimal assumptions.
\newblock {\em arXiv preprint arXiv:2409.18959}.

\bibitem[Li et~al., 2024c]{li2024unified}
Li, R., Di, Q., and Gu, Q. (2024c).
\newblock Unified convergence analysis for score-based diffusion models with deterministic samplers.
\newblock {\em arXiv preprint arXiv:2410.14237}.

\bibitem[Lu et~al., 2022]{lu2022dpm}
Lu, C., Zhou, Y., Bao, F., Chen, J., Li, C., and Zhu, J. (2022).
\newblock Dpm-solver: A fast ode solver for diffusion probabilistic model sampling in around 10 steps.
\newblock {\em Advances in Neural Information Processing Systems}, 35:5775--5787.

\bibitem[Lugmayr et~al., 2022]{lugmayr2022repaint}
Lugmayr, A., Danelljan, M., Romero, A., Yu, F., Timofte, R., and Van~Gool, L. (2022).
\newblock Repaint: Inpainting using denoising diffusion probabilistic models.
\newblock In {\em Proceedings of the IEEE/CVF conference on computer vision and pattern recognition}, pages 11461--11471.

\bibitem[Lyu et~al., 2024]{lyu2024sampling}
Lyu, J., Chen, Z., and Feng, S. (2024).
\newblock Sampling is as easy as keeping the consistency: convergence guarantee for consistency models.
\newblock In {\em Forty-first International Conference on Machine Learning}.

\bibitem[Nichol et~al., 2021]{nichol2021glide}
Nichol, A., Dhariwal, P., Ramesh, A., Shyam, P., Mishkin, P., McGrew, B., Sutskever, I., and Chen, M. (2021).
\newblock Glide: Towards photorealistic image generation and editing with text-guided diffusion models.
\newblock {\em arXiv preprint arXiv:2112.10741}.

\bibitem[Sohl-Dickstein et~al., 2015]{sohl2015deep}
Sohl-Dickstein, J., Weiss, E., Maheswaranathan, N., and Ganguli, S. (2015).
\newblock Deep unsupervised learning using nonequilibrium thermodynamics.
\newblock In {\em International Conference on Machine Learning}, pages 2256--2265. PMLR.

\bibitem[Song et~al., 2020a]{song2020denoising}
Song, J., Meng, C., and Ermon, S. (2020a).
\newblock Denoising diffusion implicit models.
\newblock {\em arXiv preprint arXiv:2010.02502}.

\bibitem[Song et~al., 2023]{song2023consistency}
Song, Y., Dhariwal, P., Chen, M., and Sutskever, I. (2023).
\newblock Consistency models.
\newblock {\em arXiv preprint arXiv:2310.02279}.

\bibitem[Song et~al., 2021]{song2021maximum}
Song, Y., Durkan, C., Murray, I., and Ermon, S. (2021).
\newblock Maximum likelihood training of score-based diffusion models.
\newblock {\em Advances in neural information processing systems}, 34:1415--1428.

\bibitem[Song and Ermon, 2019]{song2019generative}
Song, Y. and Ermon, S. (2019).
\newblock Generative modeling by estimating gradients of the data distribution.
\newblock {\em Advances in neural information processing systems}, 32.

\bibitem[Song et~al., 2020b]{song2020sliced}
Song, Y., Garg, S., Shi, J., and Ermon, S. (2020b).
\newblock Sliced score matching: A scalable approach to density and score estimation.
\newblock In {\em Uncertainty in Artificial Intelligence}, pages 574--584. PMLR.

\bibitem[Song et~al., 2020c]{song2020score}
Song, Y., Sohl-Dickstein, J., Kingma, D.~P., Kumar, A., Ermon, S., and Poole, B. (2020c).
\newblock Score-based generative modeling through stochastic differential equations.
\newblock In {\em International Conference on Learning Representations}.

\bibitem[Xu and Chi, 2024]{xu2024provably}
Xu, X. and Chi, Y. (2024).
\newblock Provably robust score-based diffusion posterior sampling for plug-and-play image reconstruction.
\newblock {\em arXiv preprint arXiv:2403.17042}.

\bibitem[Zhang and Chen, 2022]{zhang2022fast}
Zhang, Q. and Chen, Y. (2022).
\newblock Fast sampling of diffusion models with exponential integrator.
\newblock {\em arXiv preprint arXiv:2204.13902}.

\bibitem[Zhou et~al., 2024]{zhou2024fast}
Zhou, Z., Chen, D., Wang, C., and Chen, C. (2024).
\newblock Fast {ODE}-based sampling for diffusion models in around 5 steps.
\newblock In {\em Proceedings of the IEEE/CVF Conference on Computer Vision and Pattern Recognition}, pages 7777--7786.

\end{thebibliography}






\newpage
\tableofcontents
\newpage

\newpage
\appendix

\section{Actual counterpart of our algorithm}
As discussed in the paper, below we provide the true counterpart of our multi-step consistency sampling Algorithm \ref{alg:1} which involves using the true consistency function $f$ as against $\hat{f}$ and thereby leads us to the true distribution. Note, since it is using the true consistency function, it follows the true distribution $p_{t_1,...,t_K}$ as against the $\hat{p}_{t_1,...,t_K}$ and can also be treated as a \textit{True Reverse Process}.
\noindent
\begin{algorithm}[H]
\caption{Multi-Step Consistency Generation}
\label{alg:2}
\begin{algorithmic}[1]
\STATE Sample ${\rvx}_K \sim {p}_{t_K}$
\FOR{$k = K, K-1, \ldots, 1$}
    \STATE ${\rvx}'_{k-1} = {f}(t'_{k-1}, t_k, {\rvx}_K)$ \\
    \STATE Sample $z \sim \mathcal{N}(0, I_d)$
    \STATE ${\rvx}_{k-1} = e^{t'_{k-1} - t_{k-1}} {\rvx}'_{k-1} + \sqrt{1 - e^{2(t'_{k-1} - t_{k-1})}}\, z$ 
\ENDFOR
\STATE \textbf{Output} ${\rvx}_0$
\STATE ${p}_{t_0}$ denotes the density of ${\rvx}_0$
\STATE ${p}_{t_0,.., t_K}$ denotes joint density of $({\rvx}_0, {\rvx}_1, \ldots, {\rvx}_K)$
\end{algorithmic}
\end{algorithm}

\section{Proof of Theorem~\ref{thm:smooth_score_case}}

\label{appendix:proof_of_thm1}
\noindent


\noindent
Here, we provide the proof of our first main result. As highlighted in the high-level proof in the main paper, we bound the two errors: the error due to empirical PF ODE
and the initialization error. We first discuss bounding the former below.

\subsection{Error due to the empirical PF ODE (eq. \ref{empirical_pf_ode}).} 
\label{err_1}
As discussed in the algorithms provided in the main, we will denote the joint distribution of the true and approximate process by $p_{t_1, t_2, ..., t_{K}}$ and $\hat{p}_{t_1, t_2, ..., t_{K}}$ respectively. The overall idea is to bound the KL between the outputs using the data processing inequality and bounding the KL between $\hat{p}_{t_1, t_2, ..., t_{K}}$ and $p_{t_1, t_2, ..., t_{K}}$ which can be done by rewriting them using transition (conditional) probabilities. The following two lemmas describe this idea. \\

\noindent
\textbf{Notational Remark.} As discussed in the Subsection \ref{subsec_3_notations} of the main paper, we will use the notations $\rvx_{t_k}$ and $\rvx_{k}$ interchangeably both corresponding to true (resp. empirical) PF ODE at time $t_k$ (resp. $t'_{k-1}$) for a given sequence $\{t_k\}$ (resp. $\{t'_k\}$).

\vspace{0.2in}

\begin{lemma}
\label{lemma_first_kl_bw_conditionals}
 Denoting $\hat{p}_{k-1|k}$ be the conditional probability of
  $\hat{\rvx}_{k-1}$  given $\hat{\rvx}_k$, and let
  ${p}_{k-1|k}$ be the conditional probability of
  $\rvx_{k-1}$ given $\rvx_k$. Then
  \[
    \KL{{p}_{k-1|k}(\cdot|{\rvx}_{k})}{\hat{p}_{k-1|k}(\cdot|{\rvx}_k)}
    =  e^{2(t_{k-1}'-t_{k-1})}
    \frac{\|
      f(t_{k-1}',t_k,{\rvx}_k)-\hat{f}(t_{k-1}',t_k, {\rvx}_k)\|_2^2}
    {2 (1  -e^{2(t_{k-1}'-t_{k-1})})}     
  \]
\end{lemma}
\begin{proof}
For this, we know that from Algorithm~\ref{alg:2} that the conditional ${p}_{k-1|k}(\cdot|{\rvx}_{k})$ is the following Gaussain: \[
{p}_{k-1|k}(\cdot|{\rvx}_{k}) \sim \mathcal{N}\left(e^{t_{k-1}' - t_{k-1}} f(t_{k-1}', t_k, \rvx_k), \left(1 - e^{2(t_{k-1}' - t_{k-1})}\right) I_d \right)
\]
where $I_d$ is the d-dimensional identity matrix. Similarly, from Algorithm \ref{alg:1}:
\[
\hat{p}_{k-1|k}(\cdot|{\rvx}_{k}) \sim \mathcal{N}\left(e^{t_{k-1}' - t_{k-1}} \hat{f}(t_{k-1}', t_k, \rvx_k), \left(1 - e^{2(t_{k-1}' - t_{k-1})}\right) I_d \right)
\]
Now, since the covariance matrices are same for both, we can just use the following formulae for calculating KL between two gaussians with different means but same variance:
\[
\KL{{p}_{k-1|k}(\cdot|{\rvx}_{k})}{{p}_{k-1|k}(\cdot|{\rvx}_k)} = \frac{1}{2}(\mu_1-\mu_2)^T\sum(\mu_1-\mu_2)  
\]
where $\mu_1$, $\mu_2$ corresponds to the mean of the two distributions and $\sum$ corresponds to their covariance. For this case, we have:
\begin{align*}
\mu_1 &= e^{t_{k-1}' - t_{k-1}} f(t_{k-1}', t_k, \rvx_k) \\
\mu_2 &= e^{t_{k-1}' - t_{k-1}} \hat{f}(t_{k-1}', t_k, \rvx_k) \\
\Sigma &= \left(1 - e^{2(t_{k-1}' - t_{k-1})} \right) I_d
\end{align*}
Merely substituting these values in the KL formulae will lead to the desired term. \\
\end{proof}

\noindent
We will now utilize this expression to bound the KL between outputs of Algorithms \ref{alg:1} and \ref{alg:2} using the following lemma:

\vspace{0.2in}
\begin{lemma}
\label{lemma_second_KL_data_processing_ineq}
  We have
  \begin{align*}
    \KL{{p}_{t_0}}{\hat{p}_{t_0}}
    \leq & \KL{{p}_{t_1, t_2, ..., t_{K}}}{\hat{p}_{t_1, t_2, ..., t_{K}}}\\
    =& \KL{{p}_{t_K}}{\hat{p}_{t_K}} + 
   \E_{{p}_{t_1,..,t_{K}}} \left[ \sum_{k=1}^K
       \KL{{p}_{k-1|k}(\cdot|{\rvx}_{k})}{\hat{p}_{k-1|k}(\cdot|{\rvx}_k)}\right]
  \end{align*}
\end{lemma}
\begin{proof}
Since we know that LHS corresponds to first marginalizing the corresponding joint distributions and then calculating KL and RHS is the KL div between the joint distributions. Using data processing inequality, it is straightforward to argue that the inequality holds. The second equation is just decomposing the KL of the joint distribution into conditionals which can be easily verified by merely writing the RHS expression using the KL formulae. 
\end{proof}

\noindent
We now provide another lemma which will be useful in relating the expected difference between the true and approximate process generated from empirical PF ODE with the corresponding consistency functions.

\vspace{0.2in}
\begin{lemma}
\label{change_of_var_lemma}
    Given deterministic functions $g$ and $\hat{g}$ on a random variable $\rvy$ and also given random variables $\rvx,\hat{\rvx}$
 such that $\rvx=g(\rvy)$ and $\hat{\rvx}=\hat{g}(\rvy)$, then we have:
 \[
 \mathbb{E}_{\rvx,\hat{\rvx}}\|\rvx-\hat{\rvx}\|^2 = \mathbb{E}_{\rvy}\|g(\rvy)-\hat{g}(\rvy)\|^2
 \]
 \end{lemma}
 \begin{proof}
     We can write expectation of a deterministic function of random variables: 
     \[
     \mathbb{E}_{\rvx,\hat{\rvx}}[h(\rvx,\hat{\rvx}] = \int\int h(\rvx,\hat{\rvx})f_{\rvx,\hat{\rvx}}(\vx,\hat{\vx})d\vx d\hat{\vx}
     \]
     where $f_{\rvx,\hat{\rvx}}$ is the joint distribution of the two random variables. Now, since we know that $\rvx$, $\hat{\rvx}$ are not independent and each correspond to a deterministic function of some random variable $\rvy$ where $\rvx=g(\rvy)$ and $\hat{\rvx}=\hat{g}(\rvy)$. Thus, if we know $\rvy$, we have the value of $\rvx$, $\hat{\rvx}$ fixed and thus, we can use the change of variables in the previous expression:
     \[
     \mathbb{E}_{\rvx,\hat{\rvx}}[h(\vx,\hat{\vx})] = \int\int h(\vx,\hat{\vx})f_{\rvx,\hat{\rvx}}(\vx,\hat{\vx})d\vx d\hat{\vx} = \int h(g(\vy),\hat{g}(\vy))f_\rvy(\vy)d\vy = \mathbb{E}_\rvy[ h(g(\vy),\hat{g}(\vy))]
     \]
     where $f_\rvy$ is the distribution of $\rvy$ given $h$ is measurable. Now, using the choice of $h$ as $\|\cdot \|^2$, which satisfies the requirement, leads to our result. 
 \end{proof}

\vspace{0.2in}
\noindent
We now provide our most important lemma for this proof, which bounds the numerator in the RHS in Lemma \ref{lemma_first_kl_bw_conditionals} based on Young's inequality and Gronwall's inequality.

\begin{lemma}
\label{main_lemma_smooth}
  For any $\delta>0$ with $\varepsilon_{score} = O(\delta)$, we can choose $t'_{k-1}$ such that $h'_k=t_{k}-t'_{k-1}  \leq \frac{1}{2(1+L)}$, and discretization 
$h_{k-1} = t_{k}-t_{k-1} < \frac{1}{2(1+L)}$ (since $t'_{k-1}<t_{k-1}$ thus $h_k$<$h'_k$) and have:
\[
  \E_{{p}_{t_1,.., t_{K}}}  
  \|
  f(t_{k-1}',t_k,{\rvx}_k)-\hat{f}(t_{k-1}',t_k,{\rvx}_k)\|_2^2 \leq e^2({h'}_k^2\varepsilon^2_{score}) = {O}\left( \frac{ \varepsilon^2_{score}}{L^2}\right) = 
  {O}(\delta^2) .
\]
\end{lemma}
\begin{proof}
Given the definition of $f(\cdot)$ and $\hat{f}(\cdot)$ above that these correspond to the solutions of actual ODE (eq. \ref{true_pf_ode}) and empirical ODE at $t=t'_{k-1}$ (eq. \ref{empirical_pf_ode}), we have:
\begin{equation*}
    f(t_{k-1}^\prime, t_k, \rvx_k) = \rvx_{t_{k-1}^\prime},  \quad \hat{f}(t_{k-1}',t_k,{\rvx}_k) = \hat{\rvx}_{t^\prime_{k-1}}
\end{equation*}
Now, since $f$ and $\hat{f}$ are deterministic mappings being applied to $\rvy_k$ here, using Lemma \ref{change_of_var_lemma}, we can just rewrite this as:

\begin{equation*}
    \begin{aligned}
        &\E_{\rvx_k \sim {p}_{t_1,.., t_{K}}}\left\|f(t_{k-1}',t_k,{\rvx}_k)-\hat{f}(t_{k-1}',t_k,{\rvx}_k)\right\|^2 = \E_{\rvx_{t^\prime_{k-1}}, \hat{\rvx}_{t^\prime_{k-1}}}\left\|\rvx_{t^\prime_{k-1}} - \hat{\rvx}_{t^\prime_{k-1}} \right\|^2 
    \end{aligned}
\end{equation*}
Now, to bound this, we use $\Delta_t$ to denote the difference between $x_t$ and $\hat{x}_t$: $\Delta_t=\rvx_t-\hat{\rvx}_t$. Then, we have:
\begin{equation*}
\begin{aligned}
    \frac{d \|\Delta_t\|^2}{dt} = 2\langle\Delta_t, \frac{d \Delta_t}{dt}\rangle &= 2\|\Delta_t\|^2 + 2\langle\Delta_t,s_t(x_t)-\hat{s}_t(\hat{x}_t)\rangle \\
    &\leq 2\|\Delta_t\|^2 + 2\|\Delta_t\|\|s_t(\rvx_t)-\hat{s}_t(\hat{\rvx}_t)\| \\
    &\leq 2\|\Delta_t\|^2 + 2\|\Delta_t\|\left(\|\hat{\vs}_t(\rvx_t)-\hat{\vs}_t(\hat{\rvx}_t)\| +\|\vs_t({x}_t)-\hat{\vs}_t({\rvx}_t)\|\right) \\
    &\leq (2+2L)\|\Delta_t\|^2 + 2\|\Delta_t\|\left(\|\vs_t({\rvx}_t)-\hat{\vs}_t({\rvx}_t)\|\right) \\
    &\leq \left(2+2L+\frac{1}{h'_k}\right) \|\Delta_t\|^2 + h'_k\|\vs_t({\rvx}_t)-\hat{\vs}_t({\rvx}_t)\|^2 \qquad \text{(Young's Inequality)}
    \end{aligned}
\end{equation*}
where $h'_k = t_k-t'_{k-1}$ Now, using Gronwall's inequality, we have:
\begin{align*}
\mathbb{E}\left[\left\| \rvx_{t'_{k-1}} - \hat{\rvx}_{t'_{k-1}} \right\|^2\right] &\leq \exp\left(\left(2 + 2L +  \frac{1}{h'_k} \right)h'_k\right)\left(  \int_{t'_{k-1}}^{t_k}  h'_{k}\, \mathbb{E}\left[\left\|   \vs_t({\rvx}_t) - \hat{\vs}_{t}({\rvx}_{t}) \right\|^2\right] dt \right) \\
& \leq \exp\left(\left(2 + 2L +  \frac{1}{h'_k} \right)h'_k\right)\left(  \int_{t'_{k-1}}^{t_k}  h'_{k}\, \varepsilon^2_{score} dt \right) \quad \text{(using Assumption \ref{assumption_score_est})} \\
&= \exp\left(\left(2 + 2L +  \frac{1}{h'_k} \right)h'_k\right)\left(    h'^2_{k}\, \varepsilon^2_{score} \right) \quad \\
\end{align*}
The exponential part is given by:
\[
\exp\left(\left(2 + 2L + \frac{1}{h_k'}\right) h_k'\right) = \exp\left(2h_k' + 2Lh_k' + \frac{h_k'}{h_k'}\right) = \exp\left(h_k'( 2 + 2L) + 1\right).
\]
For large \( L \), the dominant term in the exponential is \( 2Lh_k' \). If \( h_k' \) does not decay sufficiently with \( L \), this term grows very rapidly. Thus, we need to control the first term in the exponential by constraining $h'_k < \frac{1}{2(1+L)}$ resulting in the overall term of the order ${O}(\frac{\varepsilon^2_{score}}{L^2})$ and we have the following final expression:

\begin{align}
     \E_{{p}_{t_1,.., t_{K}}}  
  \|
  f(t_{k-1}',t_k,{\rvx}_k)-\hat{f}(t_{k-1}',t_k,{\rvx}_k)\|_2^2 < e^2({h'}_k^2\varepsilon^2_{score})
\end{align}

\end{proof}

\noindent
\textit{\textbf{Setting} $\{t'_k\} = 0$} \textit{\textbf{to replicate the original consistency models?}} A possibility of setting the sequence $t'_k=0$ is there but then $h'_k < \frac{1}{2(1+L)}$ is not guaranteed and for this, we can just instead use $h'_k={O}(\log(1/\delta))$ but then $\varepsilon_{score}={O}(\frac{\delta^{L+1.5}}{\sqrt{log(1/\delta)}})$, where $L\geq 1$, which would thus require a very accurate estimation of score as against $\tilde{O}(L\delta)$ before. Thus, as highlighted in the main paper its adaptation to the original consistency model formulation is not straightforward.\\

\noindent
\subsection{Initialization Error.} 
\label{appendix_init_error}
If we define the forward noising process for a total time $T$ (and consequently $K$ total iterations where $\sum_{k=0}^K h_k + t_0= T$), we know that the $p_T={law}(\rvx_T)$ is still not exactly $\mathcal{N}(0,I_d)$ and is just close to it. So when we initialize the reverse/generation process with gaussian, this leads to the initialization error which is the difference between the distribution after running the forward process on the original data distribution for time $T$ and the standard gaussian distribution, which can be bounded as follows: \\

\begin{lemma}
\label{convergence_of_ou}
    (\textbf{Convergence of the OU process}). Under Assumption \ref{finite_moment}, for \( T > 1 \), we have
\[
KL(p_T \parallel \gamma_d) \leq (d + m_2) e^{-T}.
\]
where $T$ is the total time for the forward process and $m_2 = \mathbb{E}_{p_t}[\|\vx\|_2^2]$.
\end{lemma}
\begin{proof}
For any \( t > 0 \), we can use Jensen’s inequality to bound the entropy of \( p_t \):

\begin{align*}
    \int_{\mathbb{R}^d} p_t(\vx) \log p_t(\vx) \, d\vx &= \int_{\mathbb{R}^d} \left( \int_{\mathbb{R}^d} p_{t|0}(\vx|\vx_0) dP(\vx_0) \right) \log \left( \int_{\mathbb{R}^d} p_{t|0}(\vx|\vx_0) dP(\vx_0) \right) d\vx \\
&\leq \int_{\mathbb{R}^d} \left( \int_{\mathbb{R}^d} p_{t|0}(\vx|\vx_0) \log p_{t|0}(\vx|\vx_0) dP(\vx_0) \right) d\vx \\
& = \int_{\mathbb{R}^d} \left( \int_{\mathbb{R}^d} p_{t|0}(\vx|\vx_0) \log p_{t|0}(\vx|\vx_0) d\vx \right) dP(\vx_0). \\
\end{align*}

\noindent
Since for the considered OU proces, we have  \( \rvx_t | \rvx_0 \sim \mathcal{N}(\alpha_t x_0, \sigma_t^2 I_d) \), where $\sigma_t^2=1-e^{-t}$, we have
\[
\int_{\mathbb{R}^d} p_{t|0}(\vx|\vx_0) \log p_{t|0}(\vx|\vx_0) \, d\vx = -\frac{d}{2} \log(2\pi \sigma_t^2) - \frac{d}{2}.
\]
Thus,
\[
\int_{\mathbb{R}^d} p_t(\vx) \log p_t(\vx) \, d\vx \leq -\frac{d}{2} \log(2\pi \sigma_t^2) - \frac{d}{2}.
\]
Therefore,
\[
KL(p_t \parallel \gamma_d) = \int_{\mathbb{R}^d} p_t(\vx) \log p_t(\vx) \, d\vx + \mathbb{E}_{p_t} \left[ \|\vx\|^2_2 + \frac{d}{2} \log(2\pi) \right]
\leq \frac{d}{2} \log \sigma_t^{-2} + \frac{1}{2} (m_2 - d).
\]
From the exponential convergence of Langevin dynamics with a strongly log-concave stationary distribution, we obtain
\[
KL(p_T \parallel \gamma_d) \leq e^{-T + t} \left( \frac{d}{2} \log \sigma_t^{-2} + \frac{1}{2} (m_2 - d) \right).
\]
By choosing \( t = \log 2 \), we have
\[
e^t \log \left( \frac{1}{\sigma_t^2} \right) \leq 1.
\]
Thus,
\[
KL(p_T \parallel \gamma_d) \leq e^{-T} (d + m_2).
\]
\end{proof}

\noindent
\subsection{Proving Theorem \ref{thm:smooth_score_case}.} 
\label{proving_thm_3.1}
Given the above lemmas corresponding to the error components, we now provide the proof for Theorem 3.3 as follows:
\begin{proof}
From {Lemma \ref{lemma_second_KL_data_processing_ineq}}, we have:
\begin{align*}
    \KL{{p}_{t_0}}{\hat{p}_{t_0}}
    \leq & \KL{{p}_{t_1, t_2, ..., t_{K}}}{\hat{p}_{t_1, t_2, ..., t_{K}}}\\
    =& \KL{{p}_{t_K}}{\hat{p}_{t_K}} + 
   \E_{{p}_{t_1,..,t_{K}}} \left[ \sum_{k=1}^K
       \KL{{p}_{k-1|k}(\cdot|{\rvx}_{k})}{\hat{p}_{k-1|k}(\cdot|{\rvx}_k)}\right]
\end{align*}

\noindent
From {Lemma \ref{lemma_first_kl_bw_conditionals}}, the conditional KL divergence between ${p}_{k-1|k}$ and $\hat{p}_{k-1|k}$ is given by:
\[
    \KL{{p}_{k-1|k}(\cdot|{\rvx}_{k})}{\hat{p}_{k-1|k}(\cdot|{\rvx}_k)}
    =  e^{2(t_{k-1}'-t_{k-1})}
    \frac{\|
      f(t_{k-1}',t_k,{\rvx}_k)-\hat{f}(t_{k-1}',t_k, {\rvx}_k)\|_2^2}
    {2 (1  -e^{2(t_{k-1}'-t_{k-1})})}     
  \]
Substituting this into the sum in {Lemma \ref{lemma_second_KL_data_processing_ineq}}, we get:
\[
\E_{{p}_{t_1,..,t_{K}}} \left[ \sum_{k=1}^K
       \KL{{p}_{k-1|k}(\cdot|{\rvx}_{k})}{\hat{p}_{k-1|k}(\cdot|{\rvx}_k)}\right] = 
\sum_{k=1}^K \frac{e^{2(t_k' - t_k)}}{2(1 - e^{2(t_k' - t_k)})} \mathbb{E}_{{p}_{t_1,..,t_{K}}} \| f(t_k', t_k, \rvx_k) - \hat{f}(t_k', t_k, \rvx_k) \|_2^2.
\]
From {Lemma \ref{main_lemma_smooth}}, we know:
\[
  \E_{{p}_{t_1,.., t_{K}}}  
  \|
  f(t_{k-1}',t_k,{\rvx}_k)-\hat{f}(t_{k-1}',t_k,{\rvx}_k)\|_2^2  < e^2({h'}_k^2\varepsilon_{score}) =  {O}\left( \frac{ \varepsilon^2_{score}}{L^2}\right) 
\]
when $h_k < h'_k < \frac{1}{2(L+1)}$. Let us denote the upper bound on this term as $Q =e^{2}{h'}_k^2\varepsilon_{score}^2$ for all $k$.  Therefore, we now have:
\[
E_{{p}_{t_1,..,t_{K}}} \left[ \sum_{k=1}^K
       \KL{{p}_{k-1|k}(\cdot|{\rvx}_{k})}{\hat{p}_{k-1|k}(\cdot|{\rvx}_k)}\right]\leq  Q \sum_{k=1}^K\frac{e^{-2(t_k-t'_k)}}{2(1-e^{-2(t_k-t'_k)})}
\]
\end{proof}

\noindent
\textbf{Bounding $KL({p}_{t_K}\|\hat{p}_{t_K})$}. Assuming that we start from a normal distribution as an approximateand taking $h_k=\tilde{O}(1/L)$, after running for $K=TL$ iterations (with $T$ being the total time) using {Lemma \ref{convergence_of_ou}}, we have:
\[
KL(\hat{p}_{t_K}\|p_{t_K}) = KL(p_{t_K}\|\gamma^d) \leq ({d}+m_2) \text{exp}(-T)
\]
Therefore, now have the following bound:
\[
KL(p_{t_0}\|\hat{p}_{t_0}) \leq (d+m_2) e^{-T} +  Q \sum_{k=1}^K\frac{e^{-2(t_k-t'_k)}}{2(1-e^{-2(t_k-t'_k)})} \leq (d+m_2) e^{-T} +  Q \sum_{k=1}^K\frac{1}{4(t_k-t'_k)}
\]
where the last inequality uses the fact $e^x\geq 1+x$ after multiplying the numerator and denominator with $e^{2(t_k-t'_k)}$. Substituting the value of $Q$ ow we have:
\begin{align*}
    KL(p_{t_0}\|\hat{p}_{t_0}) \leq (d+m_2) e^{-T} +  e^2{h'}_k^2\varepsilon^2_{score} \sum_{k=1}^K\frac{1}{4(t_k-t'_k)}
\end{align*}
Now choosing $K=2(L+1)\log\left(\frac{L(d+m_2)}{\varepsilon_{score}}\right)$, $t_k-t'_k = \frac{1}{K}$, we have: 
\[
KL(p_{t_0}\|\hat{p}_{t_0}) \leq {O}\left(\frac{\varepsilon^2_{score}}{L^2}\cdot K^2\right) = {O}\left({\varepsilon^2_{score}}\log^2\left(\frac{L(d+m_2)}{\varepsilon_{score}}\right)\right)
\]

\vspace{0.2in}
\section{ Proof of theorem \ref{thm:non_smooth_3.3}}
\label{appendix_non_smooth}
\noindent
The proof in this part also similar to the proof in the previous section, however, since we do not have an assumption on the smoothness of the score function, we need to find an alternate way to control the $\|\vs_t(\rvx_t)-{\vs}_t(\hat{\rvx}_t)\|^2$ term in Lemma \ref{main_lemma_smooth}. For this, we take inspiration from literature \cite{chen2023improved,benton2024nearly} on SDE based diffusion analysis which analyses in the absence of smoothness assumptions. We begin by first providing a lemma adapted from \cite{chen2023improved} that bounds the Gaussian perturbation of a given probability distribution in $d-$dimension as follows. 

\subsection{Error due to empirical PF-ODE (Eq. \ref{empirical_pf_ode}): Non-smooth case}
\label{error_due_to_empirical_non_smooth}

\begin{lemma}
\label{adapted_from_improved}
\textbf{(Taken from \cite{chen2023improved})}. Let $P$ be a probability measure on $\mathbb{R}^d$. Consider the density of its Gaussian perturbation
\[
p_\sigma(\vx) \propto \int_{\mathbb{R}^d} \exp \left(- \frac{\|\vx - \vy\|^2}{2\sigma^2} \right) dP(\vy).
\]
Then for $\rvx \sim p_\sigma$, we have the sub-exponential norm bound
\[
\|\nabla^2 \log p_\sigma(\rvx)\|_{F, \psi_1} \leq \frac{d}{\sigma^2},
\]
where $\| \cdot \|_{F, \psi_1} = \| \| \cdot \|_F \|_{\psi_1}$ denotes the sub-exponential norm of the Frobenius norm of a random matrix.
\end{lemma}

\begin{proof}
We just provide a sketch here for reference. For the detailed proof please refer Lemma 12 in \cite{chen2023improved}.
First, we will have the following equation for conditional density
$\tilde{P}_\sigma(\vy|\vx)$:
\[
d\tilde{P}_\sigma(\vy|\vx) \propto \exp \left(- \frac{\|\vy - \vx\|^2}{2\sigma^2} \right) dP(\vy).
\]
Now, just writing $\nabla^2 \log p_\sigma$ in terms of $\operatorname{Var}_{\tilde{P}_\sigma(\vy|\vx)} \left( \frac{\vy}{\sigma^2} \right)$ and using the following inequality for any integer $q$:
\[
\mathbb{E}_{p_\sigma(\vx)} \left[ \| \operatorname{Var}_{\tilde{P}_\sigma(\vy|\vx)} ( \vy / \sigma^2 ) \|_F^q \right] \leq \frac{1}{\sigma^{2p}} \mathbb{E}_{p_\sigma(x)} \left[ \mathbb{E}_{\tilde{P}_\sigma(\vy|\vx)} \| (\vy - \vx)/\sigma (\vy - \vx)/\sigma^\top \|_F^q \right].
\]
and using the fact that $\frac{\rvy-\rvx}{\sigma}$ is normally distributed, we can derive the result.

\end{proof}

\noindent
We now use the above lemma to bound the expectation of our target term $\|\vs_t(\rvx_t)-{\vs}_t(\hat{\rvx}_t)\|^2$ and provide the following lemma.

\begin{lemma}
\label{expected_smooth}
 We have:
 \begin{equation}
     \mathbb{E}\|s_t(\rvx_t)-s_t(\hat{\rvx}_t)\|^2 \leq \frac{d^2}{\sigma_t^4}\mathbb{E}\left[ \|\Delta_t\|^2\exp\left(\frac{\|\Delta_t\|^2}{2\sigma_t^2}\right) \right]
 \end{equation}
 where $\Delta_t = \rvx_t-\hat{\rvx}_t$ as defined above.
\end{lemma}
\begin{proof}
    We can bound the difference using the hessian as follows:
    \[
        \vs_t(\vx_t) - {\vs}_t(\hat{\vx}_t) = \int_{0}^1 \nabla \vs_t(\vx_t+a(\hat{\vx}_t-\vx_t))(\hat{\vx}_t-\vx_t)da
    \]
    Thus, we would have:
    \[
   \mathbb{E} \|\vs_t(\rvx_t)-\vs_t(\hat{\rvx}_t)\|^2 \leq \int_{0}^1 \mathbb{E}\|\nabla \vs_t(\rvx_t+a\Delta_t)\Delta_t\|^2da
    \]

    \noindent
    Bounding the term inside the integral in the RHS using change of measure we have:
    \begin{align*}
    \mathbb{E}\|\nabla \vs_t(\rvx_t+a\Delta_t)\Delta_t\|^2 &= \mathbb{E}\left[\|\nabla \vs_t(\rvx_t)\Delta_t\|^2 \frac{d P_{\rvx_t+a\Delta_t,\Delta_t}(\rvx_t,\Delta_t)}{d P_{x_t,\Delta_t}(\rvx_t,\Delta_t)} \right] \\
    &\leq \left( \underbrace{\mathbb{E}\|\nabla \vs_t(\rvx_t)\|^4}_{T_1} \underbrace{\mathbb{E}\left( \|\Delta_t\|^2\frac{d P_{\rvx_t+a\Delta_t,\Delta_t}(\rvx_t,\Delta_t)}{d P_{\rvx_t,\Delta_t}(\rvx_t,\Delta_t)} \right)^2}_{T_2}  \right)^{1/2}
    \end{align*}
\textbf{Bounding $T_1$:} Using Lemma \ref{adapted_from_improved}. Therefore, we can now bound $T_1$ as:
\[
T_1 \leq \mathbb{E}\left(\frac{d}{\sigma_t^2}\right)^4  = \left(\frac{d}{\sigma_t^2}\right)^4
\]
\textbf{Bounding $T_2$:} We have using the data processing inequality:
\begin{align*}
    \mathbb{E}\left( \frac{d P_{\rvx_t+a\Delta_t,\Delta_t}(\rvx_t,\Delta_t)}{d P_{\rvx_t,\Delta_t}(\rvx_t,\Delta_t)} \right)^2 &= \mathbb{E}\left( \frac{d P_{\rvx_t+a\Delta_t|\Delta_t}(\rvx_t|\Delta_t)}{d P_{\rvx_t|\Delta_t}(\rvx_t|\Delta_t)} \right)^2 \\
    &\leq  \mathbb{E}\left( \frac{d P_{\rvx_t+a\Delta_t|\Delta_t,\rvx_0}(\rvx_t|\Delta_t,\rvx_0)}{d P_{\rvx_t|\Delta_t,\rvx_0}(\rvx_t|\Delta_t,\rvx_0)} \right)^2 \\
    &= \mathbb{E}\left( \frac{d P_{\rvx_t+a\Delta_t|\Delta_t,\rvx_0}(\rvx_t|\Delta_t,\rvx_0)}{d P_{\rvx_t|\rvx_0}(\rvx_t|,\rvx_0)} \right)^2
\end{align*}
Therefore, we will have:
\[
\mathbb{E}\left( \|\Delta_t\|^2 \frac{d P_{\rvx_t+a\Delta_t,\Delta_t}(\rvx_t,\Delta_t)}{d P_{\rvx_t,\Delta_t}(\rvx_t,\Delta_t)} \right)^2 \leq \mathbb{E}\left( \|\Delta_t\|^2 \frac{d P_{\rvx_t+a\Delta_t|\Delta_t,\rvx_0}(\rvx_t|\Delta_t,\rvx_0)}{d P_{\rvx_t|\rvx_0}(\rvx_t|,\rvx_0)} \right)^2
\]
Now we know that $\rvx_t+a\Delta_t|(\Delta_t,\rvx_0)\sim \mathcal{N}(\alpha_t^{-1}\rvx_0+a\Delta_t,\sigma_t^2)$ and $\rvx_t|\rvx_0\sim \mathcal{N}(\alpha_t^{-1}\rvx_0,\sigma_t^2I_d)$. Therefore, we have:
\[
\mathbb{E}\left( \|\Delta_t\|^2\frac{d P_{\rvx_t+a\Delta_t|\Delta_t,\rvx_0}(\rvx_t|\Delta_t,\rvx_0)}{d P_{\rvx_t|\rvx_0}(\rvx_t|,\rvx_0)} \right) = \mathbb{E}\left[\|\Delta_t\|^2\exp\left(\frac{a^2\|\Delta_t\|^2}{2\sigma_t^2}\right)\right] 
\]
Therefore, we have:
\begin{align*}
\mathbb{E}\|\nabla s_t(_t+a\Delta_t)\Delta_t\|^2 &\leq \left(\frac{d}{\sigma_t^2}\right)^2 \mathbb{E}\left[\|\Delta_t\|^2\exp\left(\frac{a^2\|\Delta_t\|^2}{2\sigma_t^2}\right)\right]\\
& \leq \left(\frac{d}{\sigma_t^2}\right)^2 \mathbb{E}\left[\|\Delta_t\|^2\exp\left(\frac{\|\Delta_t\|^2}{2\sigma_t^2}\right)\right]
\end{align*}
Now integrating $a$ from $0$ to $1$ gives the desired result. \\
\end{proof}

\noindent
We will now provide a version of Lemma \ref{main_lemma_smooth} which doesn't require smoothness assumption on the score function. Here, we will use the previous lemma to instead bound the target term. \\

\begin{lemma}
\label{main_lemma_non_smooth}
   For any $\delta>0$ with $\varepsilon_{score} = O(\delta)$, we can choose $t'_{k-1}$ such that $h'_k=t_{k}-t'_{k-1}  < \frac{1}{d^2}$, and consequently discretization 
$h_{k} = t_{k}-t_{k-1} < \frac{1}{d^2}$ (since $t'_{k-1}<t_{k-1}$ thus $h_k$<$h'_k$) and have:
\[
  \E_{p_{t_1,...,t_K}}  
  \|
  f(t_{k-1}',t_k,{\rvx}_k)-\hat{f}(t_{k-1}',t_k,{\rvx}_k)\|_2^2  \leq e^4{h'}_k^2\varepsilon^2_{score} 
\]
\end{lemma}
\begin{proof}
Similar to proof of Lemma \ref{main_lemma_smooth} we begin with:
\begin{equation*}
    f(t_{k-1}^\prime, t_k, \rvx_k) = \vx_{t_{k-1}^\prime},  \hat{f}(t_{k-1}',t_k,{\rvx}_k) = \hat{\vx}_{t^\prime_{k-1}}
\end{equation*}
Now, since $f$ and $\hat{f}$ are deterministic mappings being applied to $\rvy_k$ here, using {Lemma \ref{change_of_var_lemma}}, we can just rewrite this as:

\begin{equation*}
    \begin{aligned}
        &\E_{\rvx_k \sim {p}_{t_1,.., t_{K}}}\left\|f(t_{k-1}',t_k,{\rvx}_k)-\hat{f}(t_{k-1}',t_k,{\rvx}_k)\right\|^2 = \E_{\rvx_{t^\prime_{k-1}}, \hat{\rvx}_{t^\prime_{k-1}}}\left\|\rvx_{t^\prime_{k-1}} - \hat{\rvx}_{t^\prime_{k-1}} \right\|^2 
    \end{aligned}
\end{equation*}
Now, to bound this, we use $\Delta_t$ to denote the difference between $x_t$ and $\hat{x}_t$: $\Delta_t=x_t-\hat{x}_t$. Then, we have the following differential equation based on the evolution of $\rvx_t$ and $\hat{\rvx}_t$:
\[
\frac{d \|\Delta_t\|^2}{dt} = 2\langle\Delta_t, \frac{d \Delta_t}{dt}\rangle = 2\|\Delta_t\|^2 + 2\langle\Delta_t,\vs_t(\rvx_t)-\hat{\vs}_t(\hat{\rvx}_t)\rangle \\
\]
Taking expecation w.r.t. $\rvx_k$ and then using Lemma \ref{change_of_var_lemma}, we have:
\begin{equation*}
\begin{aligned}
\mathbb{E}_{\rvx_k} \frac{d \|\Delta_t\|^2}{dt}  = \frac{d \mathbb{E}_{\rvx_k}\|\Delta_t\|^2}{dt}  &= 2\mathbb{E}_{\rvx_k}\|\Delta_t\|^2 + 2\mathbb{E}_{\rvx_k}\langle\Delta_t,\vs_t(\rvx_t)-\hat{\vs}_t(\hat{\rvx}_t)\rangle \\
    \end{aligned}
\end{equation*}

\noindent
Using Lemma \ref{change_of_var_lemma}, this can be further written as:
\begin{equation*}
\begin{aligned}
&\mathbb{E}_{\rvx_k}\frac{d \|\Delta_t\|^2}{dt} \\
&= 2\mathbb{E}\|\Delta_t\|^2 + 2\mathbb{E}\langle\Delta_t,\vs_t(\rvx_t)-\hat{\vs}_t(\hat{\rvx}_t)\rangle \\
    &\leq 2\mathbb{E}\|\Delta_t\|^2 + 2\mathbb{E}\left[\|\Delta_t\|\|\vs_t(\rvx_t)-\hat{\vs}_t(\hat{\rvx}_t)\| \right] \\
    &\leq 2\mathbb{E}\|\Delta_t\|^2 + 2\mathbb{E}\left[\|\Delta_t\|\left(\|\vs_t(\rvx_t)-{\vs}_t(\hat{\rvx}_t)\| +\|\vs_t(\hat{\rvx}_t)-\hat{\vs}_t(\hat{\rvx}_t)\|\right) \right]\\
    &\leq (2+\frac{1}{h'_k})\mathbb{E}\|\Delta_t\|^2 + 2\mathbb{E}\left[\|\Delta_t\|\left(\|\vs_t(\hat{\rvx}_t)-\hat{\vs}_t(\hat{\rvx}_t)\|\right)\right] + \mathbb{E}\left[\|\vs_t(\rvx_t)-{\vs}_t(\hat{\rvx}_t)\|^2\right] \qquad \text{(Young's Inequality)} \\
    &\leq (2+\frac{1}{h'_k} + \frac{1}{h'_k})\mathbb{E}\|\Delta_t\|^2 + h'_k\mathbb{E}\left[\|\vs_t(\hat{\rvx}_t)-\hat{\vs}_t(\hat{\rvx}_t)\|^2\right] + \mathbb{E}\left[\|\vs_t(\rvx_t)-{\vs}_t(\hat{\rvx}_t)\|^2\right] \qquad \text{(Young's Inequality)} \\
    &\leq \left(2+\frac{1}{h'_k}+\frac{1}{h'_k}\right)\mathbb{E} \|\Delta_t\|^2 + h'_k\mathbb{E}\|\vs_t(\hat{\rvx}_t)-\hat{\vs}_t(\hat{\rvx}_t)\|^2 + \frac{d^2h'_k}{\sigma_t^4}\mathbb{E}\left[\|\Delta_t\|^2\exp\left(\frac{\|\Delta_t\|^2}{2\sigma_t^2}\right) \right] \qquad \text{(Lemma \ref{expected_smooth})}\\
    &\leq \left(2+\frac{1}{h'_k}+\frac{1}{h'_k}\right)\mathbb{E} \|\Delta_t\|^2 + h'_k\varepsilon^2_{score} + \frac{d^2h'_k}{\sigma_t^4}\mathbb{E}\left[\|\Delta_t\|^2\exp\left(\frac{\|\Delta_t\|^2}{2\sigma_t^2}\right) \right] \qquad \qquad  \qquad\text{(Assumption \ref{assumption_score_est})}\\
    &=  \left(2+\frac{1}{h'_k}+\frac{1}{h'_k}\right)\mathbb{E} \|\Delta_t\|^2 + h'_k\varepsilon^2_{score} + \frac{d^2h'_k}{\sigma_t^4}\mathbb{E}\left[\|\Delta_t\|^2 \right]+ \frac{d^2h'_k}{\sigma_t^4}\mathbb{E}\left[\|\Delta_t\|^2\left(\exp\left(\frac{\|\Delta_t\|^2}{2\sigma_t^2}\right) -1 \right)\right]\\
    &\leq \left(2+\frac{1}{h'_k}+\frac{1}{h'_k}\right)\mathbb{E} \|\Delta_t\|^2 + h'_k\varepsilon^2_{score} + \frac{2d^2h'_k}{\sigma_t^4}\mathbb{E}\left[\|\Delta_t\|^2 \right]\\
    \end{aligned}
\end{equation*}

\noindent
where $h'_k = t_k-t'_{k-1}$. Now, applying Gronwall's inequality will result in:

\begin{align*}
\mathbb{E}\left[\left\| x_{t'_{k-1}} - \hat{x}_{t'_{k-1}} \right\|^2\right]
& \leq \exp\left(\left(2 + \frac{2d^2h'_k}{\sigma_{t'_{k-1}}^4} +  \frac{2}{h'_k} \right)h'_k\right)\left(  \int_{t'_{k-1}}^{t_k}  h'_{k}\, \varepsilon^2_{score} dt \right) \quad  \\
&\leq \exp\left(\left(2 + \frac{2d^2h'_k}{\sigma_{t'_{k-1}}^4} +  \frac{2}{h'_k} \right)h'_k\right)\left(    h'^2_{k}\, \varepsilon^2_{score} \right) \quad \\
\end{align*}
The exponential part is given by:
\[
 \exp\left(\left(2 + \frac{2d^2h'_k}{\sigma_{t'_{k-1}}^4} +  \frac{2}{h'_k} \right)h'_k\right) = \exp\left(2h_k' +  \frac{2d^2h'^2_k}{\sigma_{t'_{k-1}}^4} + 2\right).
\]
For large \( d \), the dominant term in the exponential is \( \frac{2d^2h'^2_k}{\sigma_{t'_{k-1}}^4} \). If \( h'^2_k \) does not decay sufficiently with \(\frac{d^2}{\sigma_{t'_{k-1}}^4} \), this term grows very rapidly. Thus, we need to control the first term in the exponential by constraining $h'_k < \frac{\sigma_{t'_{k-1}}^2}{d}$ resulting in
\begin{align}
    \E_{p_{t_1,...,t_K}}  
  \|
  f(t_{k-1}',t_k,{\rvx}_k)-\hat{f}(t_{k-1}',t_k,{\rvx}_k)\|_2^2 \leq e^4 {h'}_k^2\varepsilon^2_{score}
\end{align}
and thus,  the overall term is of the order ${O}(\frac{\sigma_{t'_{k-1}}^4\varepsilon^2_{score}}{d^2})$. \\

\end{proof}

\subsection{Proving Theorem \ref{thm:non_smooth_3.3}} 
\label{proving_thm_3.3}
Now, using the above lemmas we provide the proof of the Theorem \ref{thm:non_smooth_3.3}, which is quite similar in structure to Theorem \ref{thm:smooth_score_case} proof discussed in the previous section. 



\begin{proof}

From {Lemma \ref{lemma_second_KL_data_processing_ineq}}, we have:
\begin{align*}
    \KL{{p}_{t_0}}{\hat{p}_{t_0}}
    \leq & \KL{{p}_{t_1, t_2, ..., t_{K}}}{\hat{p}_{t_1, t_2, ..., t_{K}}}\\
    =& \KL{{p}_{t_K}}{\hat{p}_{t_K}} + 
   \E_{{p}_{t_1,..,t_{K}}} \left[ \sum_{k=1}^K
       \KL{{p}_{k-1|k}(\cdot|{\rvx}_{k})}{\hat{p}_{k-1|k}(\cdot|{\rvx}_k)}\right]
\end{align*}

\noindent
From {Lemma \ref{lemma_first_kl_bw_conditionals}}, the conditional KL divergence between ${p}_{k-1|k}$ and $\hat{p}_{k-1|k}$ is given by:
\[
    \KL{{p}_{k-1|k}(\cdot|{\rvx}_{k})}{\hat{p}_{k-1|k}(\cdot|{\rvx}_k)}
    =  e^{2(t_{k-1}'-t_{k-1})}
    \frac{\|
      f(t_{k-1}',t_k,{\rvx}_k)-\hat{f}(t_{k-1}',t_k, {\rvx}_k)\|_2^2}
    {2 (1  -e^{2(t_{k-1}'-t_{k-1})})}     
  \]
Substituting this into the sum in {Lemma \ref{lemma_second_KL_data_processing_ineq}}, we get:
\[
\E_{{p}_{t_1,..,t_{K}}} \left[ \sum_{k=1}^K
       \KL{{p}_{k-1|k}(\cdot|{\rvx}_{k})}{\hat{p}_{k-1|k}(\cdot|{\rvx}_k)}\right] = 
\sum_{k=1}^K \frac{e^{2(t_k' - t_k)}}{2(1 - e^{2(t_k' - t_k)})} \mathbb{E}_{{p}_{t_1,..,t_{K}}} \| f(t_k', t_k, \rvx_k) - \hat{f}(t_k', t_k, \rvx_k) \|_2^2.
\]
From {Lemma \ref{main_lemma_non_smooth}} , we know that for $t'_k\geq\delta>0$:
\[
\E_{{p}_{t_1,.., t_{K}}}  
  \|
  f(t_{k-1}',t_k,{\rvx}_k)-\hat{f}(t_{k-1}',t_k,{\rvx}_k)\|_2^2 \leq e^4{h'}_k^2\varepsilon^2_{score} = O\left({\frac{\varepsilon^2_{score}\sigma^4_{t'_{k-1}}}{d^2}}{}\right)
\]
 when $h_k < h'_k < \frac{\sigma^2_{t'_{k-1}}}{d}$. Let us denote the upper bound on this term as $Q=e^4{h'}_k^2\varepsilon^2_{score}$ for all $k$. \\
Therefore, we now have:
\[
E_{{p}_{t_1,..,t_{K}}} \left[ \sum_{k=1}^K
       \KL{{p}_{k-1|k}(\cdot|{\rvx}_{k})}{\hat{p}_{k-1|k}(\cdot|{\rvx}_k)}\right]
\leq  Q \sum_{k=1}^K\frac{e^{-2(t_k-t'_k)}}{2(1-e^{-2(t_k-t'_k)})}
\]


\noindent
\textbf{Bounding $KL(\hat{p}_{t_K}\|p_{t_K})$}. Assuming that we start from a normal distribution as an approximate, after running for $K=\frac{1}{h_k}$, where $h_k=t_k-t_{k-1}$ iterations (with $T$ being the total time), using {Lemma \ref{convergence_of_ou}}, we have:
\[
KL(\hat{p}_{t_K}\|p_{t_K}) = KL(p_{t_K}\|\gamma^d) \leq ({d}+m_2) \text{exp}(-T)
\]
Therefore, now have the following bound:
\[
KL(p_{t_\delta}\|\hat{p}_{t_\delta}) \leq (d+m_2) e^{-T} +  Q \sum_{k=1}^K\frac{e^{-2(t_k-t'_k)}}{2(1-e^{-2(t_k-t'_k)})} \leq (d+m_2) e^{-T} +  Q \sum_{k=1}^K\frac{1}{4{(t_k-t'_k)}}
\]
where the last inequality uses the fact $e^x\geq 1+x$ after multiplying the numerator and denominator with $e^{2(t_k-t'_k)}$. Substituting $Q$ results in:
\begin{align*}
    KL(p_{t_\delta}\|\hat{p}_{t_\delta}) \leq (d+m_2) e^{-T} +  e^4{h'}_k^2\varepsilon^2_{score} \sum_{k=1}^K\frac{1}{4{(t_k-t'_k)}}
\end{align*}
Now choosing $K=\frac{1}{h_k}\log\left(\frac{(d+m_2)}{\varepsilon^2_{score}}\right)$, where $h_k = t_k-t_{k-1}$,  and thereby $T=\log\left(\frac{(d+m_2)}{\varepsilon^2_{score}}\right)$, we have: 
\[
KL(p_{t_\delta}\|\hat{p}_{t_\delta}) \leq (d+m_2) e^{-T} +  KQ \cdot O\left(\frac{1}{t_k-t'_k} \right) \leq (d+m_2) e^{-T} +  O\left({{\frac{\varepsilon^2_{score}\sigma^4_{t'_{k-1}}}{d^2}}{}}\cdot \frac{1}{h_k} \cdot \frac{1}{t_k-t'_k}\right) 
\]
Since $h_k<h'_k$, we can substitute $h_k = O\left(\frac{\sigma^2_{t'_{k-1}}}{d}\right)$ and similarly we can also substitute $t_k-t'_k=O\left(\frac{\sigma^2_{t'_{k-1}}}{d}\right)$ it finally reduces to:
\[
KL(p_{t_\delta}\|\hat{p}_{t_\delta}) \leq {O}\left({\varepsilon^2_{score}}\log\left(\frac{(d+m_2)}{\varepsilon^2_{score}}\right)\right) = \tilde{O}(\varepsilon^2_{score})
\]

\end{proof}

\section{Proofs of Theorems and Lemmas in Section \ref{learning_subsec}}
\label{appendix:Thm3_6}

\vspace{-0.1in}
\noindent
\subsection{Error control between ODE solver step and approximate trajectory}
\label{error_control_ode_solver_approx_traj}
We first discuss a lemma which controls the error due to taking a step via some ODE solver $\phi$ and the approximate trajectory during the consistency distillation.

\vspace{0.1in}

\begin{lemma}
\label{lemma_for_gap_f_theta}
    Assuming the exponential integrator as the ODE solver $\phi$ for the consistency distillation training with some discretization $\{t_n\}_{n\in [1,N]}$, we have 
    \[
    \mathbb{E}_{\hat{q}}[\|\hat{\rvx}^\phi_{t_{n-1}}-\hat{\rvx}_{t_{n-1}}\|^2_2] = \tilde{O}(e^{h_{n-1}}L^3h_{n-1}^2d)
    \]
    where $\hat{\rvx}^\phi_{t_{n-1}}$ is a step from $\hat{\rvx}_{t_n}$ using $\Phi$, \textit{i.e.} $\hat{\rvx}^\phi_{t_{n-1}} = \hat{\rvx}_{t_{n-1}}-(t_n-t_{n-1})\hat{s}_{t_n}(x_{t_n})$, $d$ is the dimension, $h_{n-1}=t_n-t_{n-1}$ and $\hat{\rvx}_{t_{n-1}}$ corresponds to the ODE: 
    \vspace{0.1in}
    \[
    d \hat{\rvx}_t = \left(-\hat{\rvx}_t - \hat{\vs}_t(\hat{\rvx}_t)\right)dt
    \]
\end{lemma}
\begin{proof}
Since, $\hat{y}^\phi_{t_{n-1}}$ is just exponential integrator type discretization on the score function applied to the empirical PF ODE (eq. \ref{empirical_pf_ode}), it will follow the ODE:
\[
    d \hat{\rvx}^\phi_t = \left(-\hat{\rvx}^\phi_t - \hat{\vs}_{t_{k+1}}(\hat{\rvx}^\phi_{t_{k+1}})\right)dt
    \]
Now, we denote $e_t = \hat{\rvx}_t^\phi - \hat{\rvx}_t$ and for $t\in [t_{n-1},t_{n}]$ we have the corresponding ODE for its evolution as:

\[
\frac{d e_t}{dt} = \left(e_t+ \hat{\vs}_{t_{n}}(\hat{\rvx}^\phi_{t_{n}}) - \hat{\vs}_{t}(\hat{\rvx}_t)\right)
\]
Now, we have to bound: $T_1 \leq \|e_t\|^2_2$. We have:
\[
\frac{d  \|e_t\|^2_2}{dt} = 2\langle e_t, \frac{de_t}{dt}\rangle  = 2 \|e_t\|_2^2 + 2 \langle e_t, e_t+ \hat{\vs}_{t_{n}}(\hat{y}^\phi_{t_{n}}) - \hat{\vs}_{t}(\hat{\rvx}_t)\rangle
\]
Now, applying cauchy schwartz in the second term ($\langle a,b\rangle \leq \|a\|\|b\|$) and then using $2ab \leq a^2+b^2$:
\[
\frac{d  \|e_t\|^2_2}{dt} \leq  2 \|e_t\|_2^2 + 2 \| e_t \|_2 \| \hat{\vs}_{t_{n}}(\hat{\rvx}^\phi_{t_{n}}) - \hat{\vs}_{t}(\hat{\rvx}_t)\|_2 \leq \|e_t\|_2^2 + \| e_t+ \hat{\vs}_{t_{n}}(\hat{\rvx}^\phi_{t_{n}}) - \hat{\vs}_{t}(\hat{\rvx}_t)\|^2_2
\]
Now, we can observe the following form here: $u'(t) \leq \beta(t) u(t)+ \alpha(t)$ and using the gronwall inequality, we will now have $u(t) \leq u(t_0) e^{(\int \beta(s)ds)}+  \int \alpha(s)e^{\int \beta(r)dr}ds$. Utilizing this into the above equation, we have:
\[
\mathbb{E}_{p_{t_1,...,t_N}} \left [ \|\hat{\rvx}^\phi_{n-1} - \rvx_{n-1}\|^2_2\right] \leq e^{h_{n-1}} \int_{t_{n-1}}^{t_n} \mathbb{E}_{p_{t_1,...,t_N}} \left [\|\hat{\vs}_{t_{n}}(\hat{\rvx}^\phi_{t_{n}}) - \hat{\vs}_{t}(\hat{\rvx}_t)\|^2_2\right ]dt
\]
where we denote $t_n-t_{n-1} = h_{n-1}$. Now, using the smoothness of the estimated score function, we have:
Thus, we have: 
\[
\mathbb{E}_{{p}_{t_1,...,t_N}} \left [\|\hat{\vs}_{t_{n}}(\hat{\rvx}^\phi_{t_{n}}) - \hat{\vs}_{t}(\hat{\rvx}_t)\|^2_2 \right] = \mathbb{E}_{{p}_{t_1,...,t_N}}\left[\|\int_t^{t_n}\frac{\partial}{\partial_r}\hat{\vs}_r (\rvx_r)dr\|_2^2\right] \leq L^2 d h^2_{n-1} (L)
\]
This leads to the following bound:
\[
\mathbb{E}_{p_{t_1,...,t_N}} [\|\hat{\rvx}^\phi_{t_{n-1}}-\rvx_{t_{n-1}}\|^2_2] \leq e^{h_{n-1}}(L^3dh_{n-1}^2)
\]
\end{proof}

\subsection{Proof of Theorem \ref{thm:gap_f_hat_f_theta}.} \label{final_proof_thm_3.5}

Given the above lemmas, we now provide the proof of Theorem \ref{thm:gap_f_hat_f_theta} mentioned in the main paper regarding bounding the error between actual $\hat{f}$ and its estimated version $\hat{f}_\theta$.
\begin{proof}
For any $t_N=\alpha$, we know that $\hat{f}_\theta(\alpha,\alpha,\cdot)=\hat{f}(\alpha,\cdot,\cdot)$ which we can construct via design. Thus, we can rewrite it as: $\E_{q}  
  \|
  \hat{f}_\theta(t'_{n-1},t'_{n-1},{\rvx}'_{n-1})-\hat{f}_\theta(t'_{n-1},t_n,{\rvx}_n)\|_2^2$. Thus, we have:
\begin{align*}
&=  \E_{p_{t_1,...,t_N}}  
  \| 
  \hat{f}_\theta(t'_{n-1},t'_{n-1},{\rvx}'_{n-1})-\hat{f}_\theta(t'_{n-1},t_{k},{\rvx}_n)\|_2^2  \\
  &=  \E_{p_{t_1,...,t_N}}  
  \| 
  \hat{f}_\theta(t'_{n-1},t'_{n-1},{\rvx}'_{n-1}) - \hat{f}_\theta(t'_{n-1},t_{n-1},\hat{\rvx}^\phi_{n-1}) + \hat{f}_\theta(t'_{n-1},t_{n-1},\hat{\rvx}^\phi_{n-1})-\hat{f}_\theta(t'_{n-1},t_n,{\rvx}_n)\|_2^2 
\end{align*}
where $y^\phi_{t_n}$ implies taking a step via the given ODE solver at time $t_n$. Assuming exponential integrator, in this setup, we can have: $\hat{\rvx}^\phi_{n-1} = e^{t_{n}-t_{n-1}}\rvx_{n}+ (e^{t_n-t_{n-1}}-1)\vs_\phi(\cdot)$. Now, we will bound the square root of this term to utilize the triangular inequality as follows:
\begin{align*}
    &=  \left(\E_{p_{t_1,...,t_N}}  
  \| 
  \hat{f}_\theta(t'_{n-1},t'_{n-1},{\rvx}'_{n-1}) - \hat{f}_\theta(t'_{n-1},t_{n-1},\hat{\rvx}^\phi_{n-1}) + \hat{f}_\theta(t'_{n-1},t_{n-1},\hat{\rvx}^\phi_{n-1})-\hat{f}_\theta(t'_{n-1},t_n,{\rvx}_n)\|_2^2 \right)^{1/2} \\
  &\leq   {( \E_{p_{t_1,...,t_N}}  
  \| 
  \hat{f}_\theta(t'_{n-1},t'_{n-1},{\rvx}'_{n-1}) - \hat{f}_\theta(t'_{n-1},t_{n-1},\hat{\rvx}^\phi_{n-1})\|_2^2} )^{1/2}  \\
  &\qquad + {(\E_{p_{t_1,...,t_N}}  
  \|\hat{f}_\theta(t'_{n-1},t_{n-1},\hat{\rvx}^\phi_{n-1})-\hat{f}_\theta(t'_{n-1},t_n,{\rvx}_n)\|_2^2})^{1/2} \\
 &\leq  \underbrace{( \E_{p_{t_1,...,t_N}}  
  \| 
  \hat{f}_\theta(t'_{n-1},t'_{n-1},{\rvx}'_{n-1}) - \hat{f}_\theta(t'_{n-1},t_{n-1},{\rvx}'_{n-1})\|_2^2}_{T_3} )^{1/2}  \\
  & \qquad + \underbrace{( \E_{p_{t_1,..,t_N}}  
  \| 
  \hat{f}_\theta(t'_{n-1},t_{n-1},{\rvx}'_{n-1}) - \hat{f}_\theta(t'_{n-1},t_{n-1},\hat{\rvx}^\phi_{n-1})\|_2^2}_{T_1} )^{1/2}\\ & \qquad+ \underbrace{(\E_{p_{t_1,...,t_N}}  
  \|\hat{f}_\theta(t'_{n-1},t_{n-1},\hat{\rvx}^\phi_{n-1})-\hat{f}_\theta(t'_{n-1},t_n,{\rvx}_n)\|_2^2}_{T_2})^{1/2} \\
\end{align*}
\\

\noindent
\textbf{Bounding $T_2$.} Using \textit{Assumption \ref{cd_assumption}}, it is straightforward to bound it as follows:
\[
T_2 \leq \sum_{k=1}^n\varepsilon_{cm}(t_k-t_{k-1}) = \varepsilon_{cd}(t_n-t_1)
\] 
\textbf{Bounding $T_1$.} Using \textit{Assumption \ref{smooth_f}}, we have:
\begin{align}
    T_1 \leq  L_f \mathbb{E}_{p_{t_1,...,t_N}} \|y_{n-1}-\hat{y}^\phi_{n-1}\|_2 
\end{align}
Now, we can bound the second term in the RHS using lemma 3.8. Using  these, we can write the final bound which is as follows:
\[
\E_{p_{t_1,...,t_N}}  
  \|
  f(t_{n-1}',t_n,{\rvx}_n)-\hat{f}_\theta(t_{n-1}',t_n,{\rvx}_n)\|_2 \leq L_fe^{{h_{n-1}}/{2}}(L^{{3}/{2}}d^{{1}/{2}}h_{n-1}) + \varepsilon_{cd}(t_n-t_1)
\]
\end{proof}

\subsection{Proof for Lemma \ref{lemma_non_smooth_cf_lipschitz}.} 
\label{lemma_3.6_proof}
For the given $\rvx_{t_n}, \rvy_{t_n}$, $n \in [2, N]$, let the ODE solver solution paths using the exact score function $\vs_t(x)$ be $\{\rvx_{t_i}\}_{i=1}^n$, $\{\rvy_{t_i}\}_{i=1}^n$, where:
\begin{align}
\label{xy_evolve}
\rvx_{t_i} = e^{h_i} \rvx_{t_{i+1}} + (e^{h_i} - 1) \vs_{t_i}(\rvx_{t_{i+1}}), \quad
\rvy_{t_i} = e^{h_i} \rvy_{t_{i+1}} + (e^{h_i} - 1) \vs_{t_i}(y_{t_{i+1}}).
\end{align}
Let the solution paths with estimated score $\hat{\vs}_t(\rvx)$ be $\{\hat{\rvx}_{t_i}\}_{i=1}^n$, $\{\hat{\rvy}_{t_i}\}_{i=1}^n$ where
\[
\hat{\rvx}_{t_i} = e^{h_i} \hat{\rvx}_{t_{i+1}} + (e^{h_i} - 1) \hat{\vs}_{ t_{i+1}}(\hat{\rvx}_{t_{i+1}}), \quad
\hat{\rvy}_{t_i} = e^{h_i} \hat{\rvy}_{t_{i+1}} + (e^{h_i} - 1) \hat{\vs}_{t_{i+1}}(\hat{\rvy}_{t_{i+1}}),
\]
and $\hat{\rvx}_{t_n} = \rvx_{t_n}, \hat{\rvy}_{t_n} = \rvy_{t_n}$. Then:
\begin{align*}
\hat{f}_\theta(t_1, t_n, \rvx_{t_n}) &= \sum_{i=2}^n \left[\hat{f}_\theta(t_1, t_i, \hat{\rvx}_{t_i} ) - \hat{f}_\theta(t_1, t_{i-1}, \hat{\rvx}_{t_{i-1}})\right] + f_\theta(\hat{\rvx}_{t_1}, t_1) \\
&= \sum_{i=2}^n \left[\hat{f}_\theta(t_1, t_i, \hat{\rvx}_{t_i} ) - \hat{f}_\theta(t_1, t_{i-1}, \hat{\rvx}_{t_{i-1}})\right] + \hat{x}_{t_1} - x_{t_1} + x_{t_1}.
\end{align*}
Thus,
\begin{align*}
\|\hat{f}_\theta(t_1, t_n, \rvx_{t_n}) - \hat{f}_\theta(t_1, t_n, \rvy_{t_n})\|_2 
&\leq \sum_{i=2}^n \|\hat{f}_\theta(t_1, t_i, \hat{\rvx}_{t_i} ) - \hat{f}_\theta(t_1, t_{i-1}, \hat{\rvx}_{t_{i-1}})\|_2 \notag \\ & \quad + \sum_{i=2}^n \|\hat{f}_\theta(t_1, t_i, \hat{\rvy}_{t_i} ) - \hat{f}_\theta(t_1, t_{i-1}, \hat{\rvy}_{t_{i-1}})\|_2 \notag \\
&\quad + \|{\hat{\rvx}_{t_1} - \rvx_{t_1}}\|_2 + \|{\hat{\rvy}_{t_1} - \rvy_{t_1}}\|_2 + \|{\rvx_{t_1} - \rvy_{t_1}}\|_2
\end{align*}
Taking expectation:
\vspace{0.05in}
\begin{align*}
&\mathbb{E}\|\hat{f}_\theta(t_1, t_n, \rvx_{t_n}) - \hat{f}_\theta(t_1, t_n, \rvy_{t_n})\|_2 \\
&\leq \sum_{i=2}^n \mathbb{E}\|\hat{f}_\theta(t_1, t_i, \hat{\rvx}_{t_i} ) - \hat{f}_\theta(t_1, t_{i-1}, \hat{\rvx}_{t_{i-1}})\|_2 \notag + \sum_{i=2}^n \mathbb{E}\|\hat{f}_\theta(t_1, t_i, \hat{\rvy}_{t_i} ) - \hat{f}_\theta(t_1, t_{i-1}, \hat{\rvy}_{t_{i-1}})\|_2 \notag \\
&\qquad + \mathbb{E}\|{\hat{\rvx}_{t_1} - \rvx_{t_1}}\|_2 + \mathbb{E}\|{\hat{\rvy}_{t_1} - \rvy_{t_1}}\|_2 + \mathbb{E}\|{\rvx_{t_1} - \rvy_{t_1}}\|_2 \\
& \leq 2(t_n-t_1)\varepsilon_{cd} + \mathbb{E}\|{\hat{\rvx}_{t_1} - \rvx_{t_1}}\|_2 + \mathbb{E}\|{\hat{\rvy}_{t_1} - \rvy_{t_1}}\|_2 + \mathbb{E}\|{\rvx_{t_1} - \rvy_{t_1}}\|_2 \ \qquad \text{(Assumption \ref{cd_assumption})}
\end{align*}

\vspace{0.1in}
Now, for the second term we have the following relation from the definition of $\rvx_{t_i}$ and $\hat{\rvx}_{t_i}$:
\begin{align*}
    \hat{\rvx}_{t_1} - \rvx_{t_1}  = \hat{\rvx}_{t_2} - \rvx_{t_2} + (e^{h_1}-1)\left(\hat{\rvx}_{t_2} - \rvx_{t_2} + \hat{\vs}_{t_2}(\hat{\rvx}_{t_2})-{\vs}_{t_2}({\rvx}_{t_2})\right) 
\end{align*}
Therefore, we have:
\begin{align*}
    \hat{\rvx}_{t_1} - \rvx_{t_1} = \sum_{i=2}^n(e^{h_{i-1}}-1)\left(\hat{\rvx}_{t_i} - \rvx_{t_i} + \hat{\vs}_{t_i}(\hat{\rvx}_{t_i})-{\vs}_{t_i}({\rvx}_{t_i})\right) 
\end{align*}
which leads to:
\begin{align*}
    &\mathbb{E}\|\hat{\rvx}_{t_1} - \rvx_{t_1}\|\\
    &= \sum_{i=2}^n(e^{h_{i-1}}-1)\mathbb{E}\|\left(\hat{\rvx}_{t_i} - \rvx_{t_i} + \hat{\vs}_{t_i}(\hat{\rvx}_{t_i})-{\vs}_{t_i}({\rvx}_{t_i})\right)\|_2 \\
    &\leq \sum_{i=2}^n(e^{h_{i-1}}-1)\left(\mathbb{E}\|\hat{\rvx}_{t_i} - \rvx_{t_i}\|_2 + \mathbb{E}\| \hat{\vs}_{t_i}(\hat{\rvx}_{t_i})-{\vs}_{t_i}({\rvx}_{t_i})\|_2 \right)\\
    &\leq \sum_{i=2}^n(e^{h_{i-1}}-1)\left(\mathbb{E}\|\hat{\rvx}_{t_i} - \rvx_{t_i}\|_2 + \mathbb{E}\| {\vs}_{t_i}(\hat{\rvx}_{t_i})-{\vs}_{t_i}({\rvx}_{t_i})\|_2 + \varepsilon_{score}\right)\\
    &\leq \sum_{i=2}^n(e^{h_{i-1}}-1)\left(\mathbb{E}\|\hat{\rvx}_{t_i} - \rvx_{t_i}\|_2 + \frac{d}{\sigma_t^2}\mathbb{E}\left[\| \hat{\rvx}_{t_i} - \rvx_{t_i}\|_2 \exp\left(\frac{\|\hat{\rvx}_{t_i} - \rvx_{t_i}\|^2}{2\sigma_t^2}\right) \right]+\varepsilon_{score}\right) \qquad \text{(lemma \ref{expected_smooth})} \\
   & \leq \sum_{i=2}^nh_{i-1}\left(\mathbb{E}\|\hat{\rvx}_{t_i} - \rvx_{t_i}\|_2 + \frac{d}{\sigma_t^2}\mathbb{E}\left[\| \hat{\rvx}_{t_i} - \rvx_{t_i}\|_2 \exp\left(\frac{\|\hat{\rvx}_{t_i} - \rvx_{t_i}\|^2}{2\sigma_t^2}\right) \right]+\varepsilon_{score}\right)  \\
   & = \sum_{i=2}^nh_{i-1}\left(\mathbb{E}\|\hat{\rvx}_{t_i} - \rvx_{t_i}\|_2 + \frac{d}{\sigma_t^2}\mathbb{E}\left[\| \hat{\rvx}_{t_i} - \rvx_{t_i}\|_2 \exp\left(\frac{\|\hat{\rvx}_{t_i} - \rvx_{t_i}\|^2}{2\sigma_t^2}\right) \right]\right) + (t_n-t_1)\varepsilon_{score} \\
\end{align*}
Now, assuming that $\|\rvx_t-\hat{\rvx}_t\|$ will be small (since score estimation error should be low), we can approximately write the above as:
\begin{align*}
     \mathbb{E}\|\hat{\rvx}_{t_1} - \rvx_{t_1}\| \leq \sum_{i=2}^nh_{i-1}\cdot \frac{d}{\sigma_{t_{i-1}}^2} \cdot  \mathbb{E}\|\hat{\rvx}_{t_i} - \rvx_{t_i}\|_2 + (t_n-t_1)\varepsilon_{score}
\end{align*}


\noindent
Since we can choose arbitrarily small $h_i$ during training, using $h_i<\frac{\sigma^2_{t_i}}{d}$ results in:
\begin{align*}
     \mathbb{E}\|\hat{\rvx}_{t_1} - \rvx_{t_1}\| \leq \sum_{i=2}^n\mathbb{E}\|\hat{\rvx}_{t_i} - \rvx_{t_i}\|_2 + h_i\varepsilon_{score}
\end{align*}
which leads to:
\begin{align*}
    \mathbb{E}\|\hat{\rvx}_{t_1} - \rvx_{t_1}\| \leq (t_n-t_1)\varepsilon_{score}
\end{align*}
Similarly, we will have by using eq. \ref{xy_evolve} :
\begin{align*}
    \mathbb{E}\|\rvx_{t_1}-\rvy_{t_1}\|_2 &\leq \sum_{t=2}^n(e^{h_{i-1}}-1) \mathbb{E}\|\vs_{t_i}(\rvx_{t_i})-\vs_{t_i}(\rvy_{t_i})\| \\ &\leq \sum_{i=1}^{n}h_i\frac{d}{\sigma^2_{t_i}}\mathbb{E}\|\rvx_{t_i}-\rvy_{t_i}\|_2 \\
    &\leq n\mathbb{E}\|x_{t_n}-y_{t_n}\|
\end{align*}
where in the last inequality we have used Lemma \ref{expected_smooth} and the fact that $h_i$ is small, $x_{t_i}, y_{t_i}$ would be close. This leads to:
\begin{align*}
\mathbb{E}\|\hat{f}_\theta(t_1, t_n, \rvx_{t_n}) - \hat{f}_\theta(t_1, t_n, \rvy_{t_n})\|_2 
& \leq 2(t_n-t_1)\varepsilon_{cd} + 2\varepsilon_{score}(t_n-t_1)  + n\mathbb{E}\|\rvx_{t_n}-\rvy_{t_n}\|_2 
\end{align*}


\vspace{0.2in}

\subsection{Proof of Theorem \ref{3.8thm_gap_f_non_smooth}.} 
\label{proof_thm_3.7}
Given the above lemmas and their proofs, we now provide the proof of Theorem \ref{3.8thm_gap_f_non_smooth} mentioned in the main paper regarding bounding the error between actual ${f}$ and its estimated version $\hat{f}_\theta$ for the non-smooth score scenario. Here we will utilize the Lemma \ref{lemma_non_smooth_cf_lipschitz} and Lemma \ref{lemma_for_gap_f_theta} to bound the error.  We now discuss the proof below.

\noindent
\textbf{Notational Remark.} $\mathbb{E}$ in this part corresponds to $\mathbb{E}_{p_{t_1,..,t_k}}$. 

\vspace{-0.2in}
\begin{proof}
For any $t_N=\alpha$, we know that $\hat{f}_\theta(\alpha,\alpha,\cdot)=\hat{f}(\alpha,\cdot,\cdot)$ which we can construct via design. Thus, we can rewrite the target term as: $\E_{p_{t_1,..,t_K}}  
  \|
  \hat{f}_\theta(t'_{n-1},t'_{n-1},{\rvx}'_{n-1})-\hat{f}_\theta(t'_{n-1},t_n,{\rvx}_n)\|_2^2$ and further simplify it as follows:
  \vspace{0.1in}
\begin{align*}
& = \E  
  \| 
  \hat{f}_\theta(t'_{n-1},t'_{n-1},{\rvx}'_{n-1})-\hat{f}_\theta(t'_{n-1},t_{n},{\rvx}_n)\|_2^2 \\
 & = \E  
  \| 
  \hat{f}_\theta(t'_{n-1},t'_{n-1},{\rvx}'_{n-1}) - \hat{f}_\theta(t'_{n-1},t_{n-1},\hat{\rvx}^\phi_{n-1}) + \hat{f}_\theta(t'_{n-1},t_{n-1},\hat{\rvx}^\phi_{n-1})-\hat{f}_\theta(t'_{n-1},t_n,{\rvx}_n)\|_2^2 
\end{align*}
where $\hat{\rvx}^\phi_{t_n}$ implies taking a step via the given ODE solver at time $t_n$. Assuming the exponential integrator for discretization, in this setup, we can have: $\hat{\rvx}^\phi_{n-1} = e^{t_{n}-t_{n-1}}\rvx_{n}+ (e^{t_n-t_{n-1}}-1)\vs_\phi(\cdot)$. Now, we will bound the square root of this term to utilize the triangular inequality as follows:
\begin{align*}
    &  \left(\E  
  \| 
  \hat{f}_\theta(t'_{n-1},t'_{n-1},{\rvx}'_{n-1}) - \hat{f}_\theta(t'_{n-1},t_{n-1},\hat{\rvx}^\phi_{n-1}) + \hat{f}_\theta(t'_{n-1},t_{n-1},\hat{\rvx}^\phi_{n-1})-\hat{f}_\theta(t'_{n-1},t_n,{\rvx}_n)\|_2^2 \right)^{1/2} \\
  & \quad \leq   {( \E  
  \| 
  \hat{f}_\theta(t'_{n-1},t'_{n-1},{\rvx}'_{n-1}) - \hat{f}_\theta(t'_{n-1},t_{n-1},\hat{\rvx}^\phi_{n-1})\|_2^2} )^{1/2}  \\ & \qquad + {(\E  
  \|\hat{f}_\theta(t'_{n-1},t_{n-1},\hat{\rvx}^\phi_{n-1})-\hat{f}_\theta(t'_{n-1},t_n,{\rvx}_n)\|_2^2})^{1/2} \\
 & \quad\leq  \underbrace{( \E  
  \| 
  \hat{f}_\theta(t'_{n-1},t'_{n-1},{\rvx}'_{n-1}) - \hat{f}_\theta(t'_{n-1},t_{n-1},{\rvx}_{n-1})\|_2^2}_{T_3} )^{1/2}  \\ & \qquad + \underbrace{( \E  
  \| 
  \hat{f}_\theta(t'_{n-1},t_{n-1},{\rvx}_{n-1}) - \hat{f}_\theta(t'_{n-1},t_{n-1},\hat{\rvx}^\phi_{n-1})\|_2^2}_{T_1} )^{1/2}\\ & \qquad + \underbrace{(\E  
  \|\hat{f}_\theta(t'_{n-1},t_{n-1},\hat{\rvx}^\phi_{n-1})-\hat{f}_\theta(t'_{n-1},t_n,{\rvx}_n)\|_2^2}_{T_2})^{1/2} \\
\end{align*}

\noindent
Now, upon observing carefully we can see that $T_3$ is just the recursive term and thus, we now focus on bounding $T_1$ and $T_2$. \\

\noindent
\textbf{Bounding $T_2$.} Using \textit{Assumption \ref{cd_assumption}}, it is straightforward to bound it as follows:
\[
T_2 \leq \sum_{k=1}^n\varepsilon_{cm}(t_k-t_{k-1}) = \varepsilon_{cd}(t_n-t_1)
\] 

\noindent
\textbf{Bounding $T_1$.} Using lemma \ref{lemma_non_smooth_cf_lipschitz}, we have:
\begin{align}
    T_1 \leq  n \mathbb{E} \|\rvx_{n-1}-\hat{\rvx}^\phi_{n-1}\|_2 + 2(t_n-t_1)\left(\varepsilon_{cd} + \varepsilon_{score} \right)
\end{align}
Now, using \textbf{\ref{lemma_for_gap_f_theta}}, we can write the final bound which as follows:
\begin{align*}
\E  
  \|
  f(t_{n-1}',t_n,{\rvx}_n)-\hat{f}_\theta(t_{n-1}',t_n,{\rvx}_n)\|_2 \leq ne^{{h_{n-1}}/{2}}(L^{{3}/{2}}d^{{1}/{2}}h_{n-1}) + (t_n-t_1)\left(3\varepsilon_{cd} +2\varepsilon_{score} \right)
\end{align*}
\end{proof}

\end{document}